\documentclass[letterpaper]{article}

\usepackage{proceed2e}
\usepackage[margin=1in]{geometry}

\usepackage{times}
\usepackage{hyperref}
\usepackage{graphicx}
\usepackage{caption}
\usepackage{amsmath,amssymb,amsthm}
\usepackage{algorithm,algorithmic}
\usepackage{natbib}
\usepackage{stfloats}

\newif\ifappendix
\appendixtrue

\newtheorem{theorem}{Theorem}
\theoremstyle{definition}

\allowdisplaybreaks

\title{$\varepsilon$-{BMC}: A Bayesian Ensemble Approach to Epsilon-Greedy Exploration in Model-Free Reinforcement Learning}

\author{ 
{\bf Michael Gimelfarb} \\
\And
{\bf Scott Sanner} \\ \\
Mechanical and Industrial Engineering \\
University of Toronto \\
ON M5S 3G8, Canada \\
\And
{\bf Chi-Guhn Lee}
}

\begin{document}

\maketitle

\begin{abstract}
Resolving the exploration-exploitation trade-off remains a fundamental problem in the design and implementation of reinforcement learning (RL) algorithms. In this paper, we focus on model-free RL using the epsilon-greedy exploration policy, which despite its simplicity, remains one of the most frequently used forms of exploration. However, a key limitation of this policy is the specification of $\varepsilon$. In this paper, we provide a novel Bayesian perspective of $\varepsilon$ as a measure of the uniformity of the Q-value function. We introduce a closed-form Bayesian model update based on Bayesian model combination (BMC), based on this new perspective, which allows us to adapt $\varepsilon$ using experiences from the environment in constant time with monotone convergence guarantees. We demonstrate that our proposed algorithm, $\varepsilon$-\texttt{BMC}, efficiently balances exploration and exploitation on different problems, performing comparably or outperforming the best tuned fixed annealing schedules and an alternative data-dependent $\varepsilon$ adaptation scheme proposed in the literature.
\end{abstract}

\section{INTRODUCTION}
\label{sec:intro}

Balancing exploration with exploitation is a well-known and important problem in reinforcement learning \citep{sutton2018reinforcement}. If the behaviour policy focuses too much on exploration rather than exploitation, then this could hurt the performance in an on-line setting. Furthermore, on-policy algorithms such as SARSA or TD($\lambda$) might not converge to a good policy. On the other hand, if the exploration policy focuses too much on exploitation rather than exploration, then the state space might not be explored sufficiently and an optimal policy would not be found. 


Historically, numerous exploration policies have been proposed for addressing the exploration-exploitation trade-off in model-free reinforcement learning, including {Boltzmann exploration} and {epsilon-greedy} \citep{mcfarlane2018survey}. There, this trade-off is often controlled by one or more tuning parameters, such as $\varepsilon$ in epsilon-greedy or the temperature parameter in Boltzmann exploration. However, these parameters typically have to be handcrafted or tuned for each task in order to obtain good performance. This motivates the design of exploration algorithms that adapt their behaviour according to some measure of the learning progress. The simplest approaches adapt the tuning parameters of a fixed class of exploration policies such as epsilon-greedy \citep{tokic2010adaptive}. Other methods, such as count-based exploration \citep{thrun1992efficient,bellemare2016unifying,ostrovski2017count} and Bayesian Q-learning  \citep{dearden1998bayesian}, use specialized techniques to develop new classes of exploration policies.

However, despite the recent developments in exploration strategies, epsilon-greedy is still often the exploration approach of choice \citep{vermorel2005multi,heidrich2009interview,mnih2015human,van2016deep}. Epsilon-greedy is both intuitive and simpler to tune than other approaches, since it is completely parameterized by one parameter, $\varepsilon$. Another benefit of this policy is that it can be easily combined with more sophisticated frameworks, such as options \citep{bacon2017option}. Unfortunately, the performance of epsilon-greedy in practice is highly sensitive to the choice of $\varepsilon$, and existing methods for adapting $\varepsilon$ from data are ad-hoc and offer little theoretical justification. 

In this paper, we take a fully \emph{Bayesian} perspective on adapting $\varepsilon$ based on return data. Recent work has demonstrated the strong potential of a Bayesian approach for parameter tuning in model-free reinforcement learning  \citep{downey2010temporal}. Another key advantage of a fully Bayesian approach over heuristics is the ability to specify priors on parameters, such as the predictive inverse variance of returns, $\tau$ in this work, which are more robust to noise or temporary digressions in the learning process. In addition, our approach can be combined with other exploration policies such as Boltzmann exploration \citep{tokic2011value}. Specifically, we contribute:
\begin{enumerate}
    \item A new Bayesian perspective of expected SARSA as an $\varepsilon$-weighted mixture of two models, the
    greedy (Q-learning) bootstrap and one which averages uniformly over all Q-values (Section \ref{subsec:sarsa});
    \item A Bayesian algorithm $\varepsilon$-\texttt{BMC} (Algorithm \ref{alg:bmc}) that is robust (Section \ref{subsec:bayes_q}), general, and adapts $\varepsilon$ efficiently (Section \ref{subsec:bmc});
    \item A theoretical convergence guarantee of our proposed algorithm (Theorem \ref{theorem:convergence-bmc}).
\end{enumerate}

Empirically, we evaluate the performance of $\varepsilon$-\texttt{BMC} on domains with discrete and continuous state spaces, using tabular and approximate RL methods. We empirically show that our algorithm can outperform exploration strategies that fix or anneal $\varepsilon$ based on time, and even existing adaptive algorithms. In the end, $\varepsilon$-\texttt{BMC} is a novel, efficient and general approach to adapting the exploration parameter in epsilon-greedy policies that empirically outperforms a variety of fixed annealing schedules and other ad-hoc approaches.

\section{RELATED WORK}

Our paper falls within the scope of adaptive epsilon greedy algorithms. Perhaps the most similar approach to our work is the Value Differences Based Exploration (VDBE) algorithm of \cite{tokic2010adaptive}, in which $\varepsilon$ was modelled using a moving average and updated according to the Bellman (TD) error. However, that algorithm was presented for stationary-reward multi-armed bandits. \cite{tokic2011value} combined the ideas of VDBE with Boltzmann exploration to create the VDBE-Softmax algorithm. \cite{dos2017adaptive} later developed a similar heuristic algorithm that worked on non-stationary multi-armed bandit problems. However, all these approaches are heuristic in nature; our paper approaches the problem from a Bayesian perspective.

Modelling Q-values using normal-gamma priors, as done in our paper, is a cornerstone of Bayesian Q-learning \citep{dearden1998bayesian}. However, that paper is fundamentally different from ours, in that it addresses the problem of exploration by adding a bonus to the Q-values that estimates the myopic \emph{value of perfect information}. Our paper, on the other hand, applies the normal-gamma prior only to model the variance of the returns, while the exploration is handled using the epsilon-greedy policy with $\varepsilon$ modelled using a Beta distribution.

\section{PRELIMINARIES}
\label{sec:prelim}

\subsection{MARKOV DECISION PROCESSES}
\label{subsec:mdp}

In this paper, we denote a \emph{Markov decision process} (MDP) as a tuple $\langle \mathcal{S}, \mathcal{A}, T, R, \gamma \rangle$, where $\mathcal{S}$ is a set of states, $\mathcal{A}$ is a finite set of actions, $T : \mathcal{S} \times \mathcal{A} \times \mathcal{S} \to [0, \infty)$ is a transition function for the system state, $R : \mathcal{S} \times \mathcal{A} \times \mathcal{S} \to \mathbb{R}$ is a bounded reward function, and $\gamma \in (0, 1)$ is a discount factor.

Randomized exploration policies are sequences of mappings from states to probability distributions over actions. Given an MDP $\langle \mathcal{S}, \mathcal{A}, T, R, \gamma \rangle$, for each state-action pair $(s,a)$ and policy $\pi$, we define the expected return
\begin{equation*}
    Q^\pi(s,a) = \mathbb{E}_{\pi,T}\left[\sum_{t = 0}^\infty \gamma^t r_{t+1} \Big| s_0 = s, a_0 = a \right],
\end{equation*}
where $s_t$ is sampled from $T$, $a_t$ are sampled from $\pi(\cdot | s_t)$ and $r_{t+1} = R(s_t, a_t, s_{t+1})$. The associated value function is $V^\pi(s) = \max_{a\in \mathcal{A}} Q^\pi(s,a)$, and the objective is to learn an optimal policy $\pi^*$ that attains the supremum of $V^\pi(s)$ over all policies. \cite{puterman2014markov} contains a more detailed treatment of this subject.

\subsection{REINFORCEMENT LEARNING}
\label{subsec:rl}

In the reinforcement learning setting, neither $T$ nor $R$ are known, so optimal policies are learned from experience, defined as sequences of transitions $(s_t, a_t, r_{t+1}, s_{t+1}, a_{t+1}), \ t = 0, 1,\dots$ broken up into \emph{episodes}. Here, states and rewards are sampled from the environment, and actions follow some exploration policy $\pi$. Given an estimate $\tilde{G}_t$ of the expected return at time $t$ starting from state $s=s_t$ and taking action $a=a_t$, \emph{temporal difference (TD) learning} updates the Q-values as follows:
\begin{equation*}
    Q_{t+1}(s,a) = Q_t(s,a) + \eta_t \left(\tilde{G}_{t} - Q_t(s,a)\right),
\end{equation*}
where $\eta_t \in (0,1]$ is a problem-dependent learning rate parameter.

Typically, $\tilde{G}_{t}$ is computed by bootstrapping from the current Q-values, in order to reduce variance. Two of the most popular bootstrapping algorithms are \emph{Q-learning} and \emph{SARSA}, given respectively as:
\begin{align}
\label{eqn:q}
    \tilde{G}_{t}^{Q} &= r_{t+1} + \gamma \max_{a' \in \mathcal{A}} Q_t(s_{t+1}, a'), \\
\label{eqn:sarsa}
    \tilde{G}_{t}^{SARSA} &= r_{t+1} + \gamma Q_t(s_{t+1}, a_{t+1}).
\end{align}
Q-learning is an example of an off-policy algorithm, whereas SARSA is on-policy. Under relatively mild conditions and in tabular settings, Q-learning has been shown to converge to the optimal policy with probability one \citep{watkins1992q}.

One additional algorithm that is important in this work is \emph{expected SARSA}
\begin{equation}
\label{eqn:expected-sarsa}
    \tilde{G}_{t}^{ExpS} = r_{t+1} + \gamma \mathbb{E}_{a' \sim \pi}[Q_t(s_{t+1}, a')],
\end{equation}
which is similar to SARSA, but in which the uncertainty of the next action $a_{t+1}$ with respect to $\pi$ is averaged out. This results in considerable variance reduction as compared to SARSA, and theoretical properties of this algorithm are detailed in \cite{van2009theoretical}. \cite{sutton2018reinforcement} provides a comprehensive treatment of reinforcement learning methods.

\subsection{EPSILON-GREEDY POLICY}
\label{subsec:epsilon_greedy}

In this paper, exploration is carried out using \emph{$\varepsilon$-greedy policies}, defined formally as
\begin{equation}
\label{eqn:epsilon-greedy}
    \pi^{\varepsilon}(a | s) 
    = \begin{cases}
        1 - \varepsilon_t + \frac{\varepsilon_t}{|\mathcal{A}|} &\mbox{ if } a = \arg \max_{a'\in \mathcal{A}} Q_t(s,a')\\
         \frac{\varepsilon_t}{|\mathcal{A}|} &\mbox{ otherwise }
    \end{cases}.
\end{equation}
In other words, $\pi^{\varepsilon}$ samples a random action from $\mathcal{A}$ with probability $\varepsilon_t \in [0,1]$, and otherwise selects the greedy action according to $Q_t$. As a result, $\varepsilon_t$ can be interpreted as the relative importance placed on exploration.

The optimal value of the parameter $\varepsilon_t$ is typically problem-dependent, and found through experimentation. Often, $\varepsilon_t$ is annealed over time in order to favor exploration at the beginning, and exploitation closer to convergence \citep{sutton2018reinforcement}. However, such approaches are not adaptive since they do not take into account the learning process of the agent. In this paper, our main objective is to derive a data-driven tuning strategy for $\varepsilon_t$, that depends on current learning progress rather than trial and error.

\section{ADAPTIVE EPSILON-GREEDY}
\label{sec:main}

In this section, we show how the expected return under epsilon-greedy policies can be written as an average of two return models weighted by $\varepsilon$. There are two relevant Bayesian methods for combining multiple models based on evidence: \emph{Bayesian model averaging} (BMA), and \emph{Bayesian model combination} (BMC). Generally, the Bayesian model combination approach is preferred to model averaging, since it provides a richer space of hypotheses and reduced variance \citep{minka2000bayesian}. By interpreting $\varepsilon$ as a random variable whose posterior distribution can be updated on the basis of observed data, BMC naturally leads to a method for $\varepsilon$ adaptation.


\subsection{A BAYESIAN INTERPRETATION OF EXPECTED SARSA WITH THE EPSILON-GREEDY POLICY}
\label{subsec:sarsa}

We begin by combining the definition of expected SARSA (\ref{eqn:expected-sarsa}) with the $\varepsilon$-greedy policy (\ref{eqn:epsilon-greedy}). For $s'=s_{t+1}$, $a^* = \arg\max_{a'} Q_t(s',a')$, and $r' = r_{t+1}$ we have
\begin{align}
    &\tilde{G}_{t}^{ExpS} \nonumber \\ 
    &= r' +  \gamma \sum_{a \in \mathcal{A}} \pi^{\varepsilon} (a | s') Q_t(s',a) \nonumber \\
    &= r' + \gamma \left(1 - \varepsilon_t + \frac{\varepsilon_t}{|\mathcal{A}|}\right) Q_t(s',a^*) \nonumber + \gamma \frac{\varepsilon_t}{|\mathcal{A}|} \sum_{a \not= a^*}  Q_t(s',a) \nonumber \\
    &= r' + \left(1 - \varepsilon_t \right) \gamma Q_t(s',a^*) + \varepsilon_t \gamma \frac{1}{|\mathcal{A}|} \sum_{a \in \mathcal{A}} Q_t(s',a) \nonumber \\
    \label{eqn:expected-sarsa-experts}
    &= \left(1 - \varepsilon_t \right) \tilde{G}_{t}^{Q} + \varepsilon_t \tilde{G}_{t}^{U},
\end{align}
where $\tilde{G}_{t}^{Q}$ is the Q-learning bootstrap (\ref{eqn:q}) and 
\begin{equation}
\label{eqn:uniform}
    \tilde{G}_{t}^{U} = r_{t+1} + \gamma \frac{1}{|\mathcal{A}|} \sum_{a' \in \mathcal{A}} Q_t(s_{t+1}, a')
\end{equation}
is an estimate that uniformly averages over all the action-values, which we call the \emph{uniform} model. This leads to the following important observation: \emph{expected SARSA can be viewed as a probability-weighted average of two models, the greedy model $\tilde{G}_{t}^{Q}$ that trusts the current Q-value estimates and acts optimally with respect to them, and the uniform model $\tilde{G}_{t}^{U}$ that completely distrusts the current Q-value estimates and consequently places a uniform belief over them.} Under this interpretation, $\varepsilon_t$ and $1 - \varepsilon_t$ are the posterior beliefs assigned to the two aforementioned models, respectively. In the following subsections, we verify this simple fact algebraically in the context of Bayesian model combination. We also develop a method for maintaining the (approximate) posterior belief state efficiently, with a computational cost that is constant in both space and time, and with provable convergence guarantees.

\subsection{BAYESIAN Q-LEARNING}
\label{subsec:bayes_q}

In order to facilitate tractable learning and inference, we assume that the return observation $q_{s,a}$ at time $t$, given the model $m \in \lbrace Q, U\rbrace$, is normally distributed:
\begin{equation}
\label{eqn:bayesian-q}
\begin{aligned}
    q_{s,a} | m, \tau \sim \mathcal{N}\left(\tilde{G}_t^m, {\tau}^{-1}\right),
\end{aligned}
\end{equation}
where the means $\tilde{G}_t^Q$ and $\tilde{G}_t^U$ are given in (\ref{eqn:q}) and (\ref{eqn:uniform}), respectively, and $\tau > 0$ is the inverse of the variance, or \emph{precision}. This assumption can be justified, by viewing the return as a discounted sum of future (random) reward observations, and appealing to the central limit theorem when $\gamma$ is close to 1 and the MDP is ergodic \citep{dearden1998bayesian}. 

There are two special cases of interest in this work. In the first case, $\tau$ is allowed to be constant across all state-action pairs and models, and naturally leads to a state-independent $\varepsilon$ adaptation. This approach is particularly advantageous when it is costly or impossible to maintain independent statistics per state, such as when the state space is very large or continuous in nature. In the second case, independent statistics are maintained per state and lead to state-dependent exploration. 

In order to update $\tau$, we consider the standard normal-gamma model:
\begin{equation}
\label{eqn:gaussian-gamma-prior}
\begin{aligned}
    \mu, \tau &\sim \textrm{NormalGamma}\left(\mu_0, \tau_0, a_0, b_0\right),\\
    q_{s,a} | \mu, \tau &\sim \mathcal{N}\left(\mu, \tau^{-1}\right),
\end{aligned}
\end{equation}
where $q_{s,a}$ are i.i.d. given $\mu$ and $\tau$. Since the returns in different state-action pairs are dependent, this assumption is likely to be violated in practice. However, it leads to a compact learning representation necessary for tractable Bayesian inference, and has been used effectively in the existing literature in similar forms \citep{dearden1998bayesian}. Furthermore, (\ref{eqn:gaussian-gamma-prior}) is not used to model the Q-values directly in our paper, but rather, to facilitate robust estimation of $\tau$, as we now show.

Given data $\mathcal{D} = \lbrace q_{s_i,a_i,i} \,|\, i = 0, 1 \dots t -1 \rbrace$ of previously observed returns, the joint posterior distribution of $\mu$ and $\tau$ with likelihood (\ref{eqn:bayesian-q}) and prior (\ref{eqn:gaussian-gamma-prior}) is also normal-gamma distributed, and so the marginal posterior distribution of $\tau$, $\mathbb{P}(\tau | \mathcal{D})$, is gamma distributed with parameters:
\begin{equation}
\label{eqn:tau_posterior}
\begin{aligned}
    a_{t} &= a_0 + \frac{t}{2}, \\
    b_{t} &= b_0 + \frac{t}{2} \left( \hat{\sigma}_t^2 + \frac{\tau_0}{\tau_0 + t} (\hat{\mu}_t - \mu_0)^2 \right),
\end{aligned}
\end{equation}
where $\hat{\mu}_t$ and $\hat{\sigma}_t^2$ are the sample mean and variance of the returns in $\mathcal{D}$, respectively \citep{bishop2006prml}. These quantities can be updated online after each new observation $d'$ in constant time \citep{welford1962note}.

Finally, for each model $m \in \lbrace Q, U\rbrace$, we marginalize over the uncertainty in $\tau$, using (\ref{eqn:bayesian-q}) and (\ref{eqn:tau_posterior}) as follows:
\begin{align*}
    \mathbb{P}(q_{s,a} | m, \mathcal{D}) 
    &= \int_{0}^{\infty} \mathbb{P}(q_{s,a} | m, \tau) \mathbb{P}(\tau | \mathcal{D}) \,\mathrm{d}\tau \\
    &\propto \int_{0}^{\infty}
    \tau^{1/2} e^{-\frac{\tau}{2}(q_{s,a} - \tilde{G}_t^m)^2} \tau^{a_t-1} e^{-b_{t} \tau}  \,\mathrm{d}\tau \\
    &= \int_{0}^{\infty} \tau^{a_{t} + \frac{1}{2} - 1}  e^{- \left(b_{t} + \frac{1}{2}(q_{s,a} - \tilde{G}_t^m)^2\right) \tau} \,\mathrm{d}\tau \\
    &\propto \left(b_{t} +  \frac{1}{2}(q_{s,a} - \tilde{G}_t^m)^2\right)^{-\frac{2 a_{t} + 1}{2}}.
\end{align*}
Finally, we have:
\begin{equation}
\label{eqn:precision_out}
    \begin{aligned}
        q_{s,a} | m, \mathcal{D} &\sim \textrm{St}\left(\tilde{G}_t^m, \frac{a_{t}}{b_{t}},2 a_{t}\right),
    \end{aligned}
\end{equation}
where $\textrm{St}(\mu,\lambda,\nu)$ is the three-parameter Student t-distribution \citep{bishop2006prml}. Therefore, marginalizing over the unknown precision $\tau$ leads to a t-distributed likelihood function. Alternatively, one could simply use the Gaussian likelihood in equation (\ref{eqn:bayesian-q}) and treat $\tau$ as a problem-dependent tuning parameter. However, the heavy tail property of the t-distribution is advantageous in the non-stationary setting typically encountered in reinforcement learning applications, where Q-values change during the learning phase. We now show how to link this update with the expected SARSA decomposition (\ref{eqn:expected-sarsa-experts}) to derive an adaptive epsilon-greedy policy.

\subsection{EPSILON-{BMC}: ADAPTING EPSILON USING BAYESIAN MODEL COMBINATION}
\label{subsec:bmc}

In the general setting of \emph{Bayesian model combination}, we model the uncertainty in Q-values for each state-action pair $(s, a)$ as random variables with posterior distribution $\mathbb{P}(q_{s,a} | \mathcal{D})$. The expected posterior return can be written as an average over all possible \emph{combinations} of greedy and uniform model,
\begin{equation}
\label{eqn:bmc}
    \mathbb{E}[q_{s,a} | \mathcal{D}] = \int_0^1 \mathbb{E}[q_{s,a} | w]\ \mathbb{P}(w | \mathcal{D}) \,\mathrm{d}w,
\end{equation}
where $w$ is the weight assigned to the uniform model and $1 - w$ is the weight assigned to the greedy model \citep{monteith2011turning}. As will be verified shortly, the expectation of this weight $w$ given the past return data $\mathcal{D}$ will turn out to be a Bayesian interpretation of $\varepsilon_t$. 

The belief over $w$ is maintained as a posterior distribution $p_t(w) = \mathbb{P}(w | \mathcal{D})$. Continuing from (\ref{eqn:bmc}) and using (\ref{eqn:precision_out}):
\begin{align*}
    &\mathbb{E}[q_{s,a} | \mathcal{D}] \\
    &= \int_0^1 \mathbb{E}[q_{s,a} | w, \mathcal{D}]\ \mathbb{P}(w | \mathcal{D}) \,\mathrm{d}w \\
    &= \int_0^1 \sum_{m \in \lbrace Q, U\rbrace} \mathbb{E}[q_{s,a} | m, \mathcal{D}] \ \mathbb{P}(m | w)\ \mathbb{P}(w | \mathcal{D}) \,\mathrm{d}w \\
    &= \int_0^1 \left(\mathbb{E}[q_{s,a} | Q, \mathcal{D}] (1 - w) + \mathbb{E}[q_{s,a} | U, \mathcal{D}] w\right) \mathbb{P}(w | \mathcal{D}) \,\mathrm{d}w \\
    &=  \left(1 - \mathbb{E}[w | \mathcal{D}]\right) \mathbb{E}[ q_{s,a} | Q, \mathcal{D}] + \mathbb{E}[w | \mathcal{D}] \ \mathbb{E}[q_{s,a} | U, \mathcal{D}]  \\
    &= \left(1 - \mathbb{E}[w | \mathcal{D}]\right) \tilde{G}_{t}^{Q} + \mathbb{E}[w | \mathcal{D}] \tilde{G}_{t}^{U},
\end{align*}
which is exactly (\ref{eqn:expected-sarsa-experts}) except that now $\varepsilon_t = \mathbb{E}[w | \mathcal{D}]$. We have thus shown that \emph{the expected SARSA bootstraps with data-driven $\varepsilon_t$ can be viewed in terms of Bayesian model combination.} We denote this new estimate $\varepsilon_t^{BMC}$.

The posterior distribution $p_{t}(w) = \mathbb{P}(w | \mathcal{D})$ is updated recursively by Bayes' rule:
\begin{align*}
    p_{t}(w)
    &\propto \mathbb{P}(d' | w, \mathcal{D}) p_{t-1}(w) \\
    &\propto \sum_{m \in \lbrace Q, U \rbrace} \mathbb{P}(d' | m, \mathcal{D})\mathbb{P}(m | w) p_{t-1}(w) \\
    &\propto \left(\mathbb{P}(d' | U, \mathcal{D}) w + \mathbb{P}(d' | Q, \mathcal{D}) (1 - w) \right) p_{t-1}(w),
\end{align*}
for every new observation $d'$. Since the number of terms in $p_t$ grows exponentially in $|\mathcal{D}|$, it is necessary to use posterior approximation techniques to compute $\mathbb{E}[w | \mathcal{D}]$. 

One approach to address the intractability of computing an exact posterior $p_{t}(w)$ is to sample directly from the distribution. However, such an approach is inherently noisy and inefficient in practice. Instead, we apply the Dirichlet moment-matching technique \citep{hsu2016online,gimelfarb2018} that was shown to be effective at differentiating between good and bad models from data, easy to implement, and efficient. In particular, we apply the approach in \cite{gimelfarb2018} with the models $\tilde{G}_t^Q$ and $\tilde{G}_t^U$, by matching the first and second moments of $p_t$ to those of the beta distribution $\mathrm{Beta}(\alpha_{t},\beta_{t})$ and solving the resulting system of equations for $\alpha_{t}$ and $\beta_{t}$. The closed-form solution is:
\begin{align}        
\label{eqn:first-bmc}
    m_{t} &= \frac{\alpha_{t}}{\alpha_{t} + \beta_{t} + 1} \frac{e_t^U(\alpha_{t}+1) + e_t^Q\beta_{t}}{e_t^U\alpha_{t} + e_t^Q\beta_{t}}, \\
\label{eqn:second-bmc}
    v_{t}
    &= \frac{\alpha_{t}}{\alpha_{t} + \beta_{t} + 1} \frac{\alpha_{t} + 1}{\alpha_{t} + \beta_{t} + 2} \frac{e_t^U(\alpha_{t}+2) + e_t^Q\beta_{t}}{e_t^U\alpha_{t} + e_t^Q\beta_{t}}, \\
\label{eqn:third-bmc}
    r_{t} &= \frac{m_{t} - v_{t}}{v_{t} - m_{t}^2}, \\
\label{eqn:fourth-bmc}
    \alpha_{t+1} &= m_{t} r_{t}, \\
    \beta_{t+1} &= (1 - m_{t}) r_{t} \label{eqn:last-bmc},
\end{align}
where $e_t^U$ and $e_t^Q$ are the respective probabilities of observing a return $d' = \tilde{G}_{t}^{ExpS}$ under the distributions (\ref{eqn:precision_out}). It follows that
\begin{align}
\label{eqn:eg-bmc}
    \varepsilon_t^{BMC} \approx \mathbb{E}_{\mathrm{Beta}(\alpha_t, \beta_t)}[w | \mathcal{D}] = \frac{\alpha_t}{\alpha_t + \beta_t}.
\end{align}
All quantities, including $a_t$, $b_t$, $\alpha_t$ and $\beta_t$, can be computed online in constant time and with minimal storage overhead without caching $\mathcal{D}$. We call this approach \texttt{$\varepsilon$-BMC} and present the corresponding pseudo-code in Algorithm~\ref{alg:bmc}. Therein, lines with a * indicate additions to the ordinary expected SARSA algorithm.

\begin{algorithm}
	\caption{\texttt{$\varepsilon$-BMC} with Expected SARSA}
	\label{alg:bmc}
	\begin{algorithmic}[1]	
	    \STATE* initialize $\mu_0, \tau_0, a_0, b_0, \hat{\mu}=0, \hat{\sigma}^2=\infty,\alpha,\beta$
		\FOR{each episode}
		    \STATE initialize $s$
		    \FOR{each step in the episode}
		        \STATE* $\varepsilon \gets \frac{\alpha}{\alpha + \beta}$
		        \STATE choose action $a$ using  $\varepsilon$-greedy policy $\pi^{\varepsilon}$ (\ref{eqn:epsilon-greedy})
		        \STATE take action $a$, observe $r$ and $s'$
		        \STATE $\tilde{G}^{Q} \gets r + \gamma \max_{a'} Q(s', a')$
        		\STATE $\tilde{G}^{U} \gets r + \gamma \frac{1}{|\mathcal{A}|} \sum_{a'} Q(s', a')$
		        \STATE $\tilde{G}^{ExpS} \gets r + \gamma \sum_{a'} \pi^{\varepsilon}(a' | s') Q(s', a')$ 
		        \COMMENT{note $\tilde{G}^{ExpS} = (1 - \varepsilon) \tilde{G}^Q + \varepsilon \tilde{G}^U$}
		        \STATE $Q(s,a) \gets Q(s,a) + \eta [\tilde{G}^{ExpS} - Q(s, a)]$ 
		        \STATE* update $\hat{\mu}$ and $\hat{\sigma}^2$ using observation $\tilde{G}^{ExpS}$
		        \STATE* compute $a$ and $b$ using (\ref{eqn:tau_posterior})
		        \STATE* compute $e^Q$ and $e^U$ using (\ref{eqn:precision_out})
        		\STATE* update $\alpha$ and $\beta$ using (\ref{eqn:first-bmc})-(\ref{eqn:last-bmc})
		        \STATE $s \gets s'$ 
		    \ENDFOR
		\ENDFOR
	\end{algorithmic}
\end{algorithm}


In Algorithm \ref{alg:bmc}, the expected SARSA return $\tilde{G}^{ExpS}$ was used to update both the posterior on $\varepsilon$ and the Q-values. However, it is possible to update the Q-values in line 11 using a different estimator of the future return, including Q-learning (\ref{eqn:q}), SARSA (\ref{eqn:sarsa}), or other approaches. The Q-values could also be approximated using a deep neural network or other function approximator. The algorithm can also be run off-line by caching $\mathcal{D}$ for an entire episode and then updating $\varepsilon$. State-dependent exploration can be easily implemented by maintaining the posterior statistics independently for each state, or approximating them using a neural network.

A final advantage of our algorithm is a provable convergence guarantee under fairly general assumptions.

\begin{theorem}[Monotone Convergence of $\varepsilon$-BMC]
\label{theorem:convergence-bmc}
    Suppose $0 < \alpha_0 \leq \beta_0 < \infty$. Then, $\varepsilon_{t+1}^{BMC} \leq \varepsilon_t^{BMC}$ for all $t =  0, 1\dots$, therefore $\varepsilon_t^{BMC}$ converges as $t \to \infty$. 
\end{theorem}

\begin{proof}
    Let $\varepsilon_t = \varepsilon_t^{BMC}$ and observe that:
    \begin{align}
    \label{eqn:epsilon_decay}
        &\varepsilon_{t+1} \nonumber \\
        &= \frac{\alpha_{t+1}}{\alpha_{t+1} + \beta_{t+1}} = \frac{m_t r_t}{r_t} = m_t \nonumber \\
        &= \frac{\alpha_t}{\alpha_t + \beta_t + 1} \frac{e_t^U(\alpha_t+1) + e_t^Q \beta_t}{e_t^U\alpha_t + e_t^Q \beta_t} \nonumber \\
        &= \frac{m_{t-1} r_{t-1}}{r_{t-1} + 1} \frac{e_t^U (m_{t-1} r_{t-1} + 1) + e_t^Q (1 - m_{t-1}) r_{t-1}}{e_t^U m_{t-1} r_{t-1} + e_t^Q (1 - m_{t-1}) r_{t-1}} \nonumber \\
        &= \varepsilon_t \frac{\left(e_t^U m_{t-1} + e_t^Q (1 - m_{t-1}) \right) r_{t-1} + e_t^U}{\left(e_t^U m_{t-1} + e_t^Q (1 - m_{t-1})\right)(r_{t-1}+1)},
    \end{align}
    where the first line uses (\ref{eqn:fourth-bmc})-(\ref{eqn:eg-bmc}), the second uses (\ref{eqn:first-bmc}) and the third uses (\ref{eqn:fourth-bmc}) and (\ref{eqn:last-bmc}). Then, since $d_t = \tilde{G}_t^{ExpS} = (1 - \varepsilon_t) \tilde{G}_t^Q + \varepsilon_t \tilde{G}_t^U$:
    \begin{align*}
        (d_t - \tilde{G}_t^U)^2 &= (1 - \varepsilon_t)^2 (\tilde{G}_t^Q - \tilde{G}_t^U)^2, \\
        (d_t - \tilde{G}_t^Q)^2 &= \varepsilon_t^2 (\tilde{G}_t^Q - \tilde{G}_t^U)^2.
    \end{align*}
    Now, if $\varepsilon_t \leq \frac{1}{2}$, then this implies that
    \begin{align*}
        &\phantom{=} (d_t - \tilde{G}_t^U)^2 \geq (d_t - \tilde{G}_t^Q)^2 \implies e_t^U \leq e_t^Q \\
        &\implies e_t^U \leq e_t^U m_{t-1} + e_t^Q (1 - m_{t-1}),
    \end{align*} 
    and from (\ref{eqn:epsilon_decay}) we conclude that $\varepsilon_{t+1} \leq \varepsilon_t$. The first statement of the theorem follows from the assumption $\varepsilon_0 \leq \frac{1}{2}$ and a standard induction argument. The second statement follows from the monotone convergence theorem (see, e.g. \cite{rudin1976principles}, pg. 56).
\end{proof}

The convergence of $\varepsilon$-\texttt{BMC} holds using any value function representation, including neural networks. It is important to note that convergence of $\varepsilon$-\texttt{BMC} can only be guaranteed when $\varepsilon$ is initialized in $[0, 0.5]$. However, this is not a concern in practice, since it has been found that there is no significant gain in using values of $\varepsilon$ larger than 0.5 \citep{dos2017adaptive}.

\section{EMPIRICAL EVALUATION}
\label{sec:numerical}

To demonstrate the ability of Algorithm~\ref{alg:bmc} to adapt in a variety of environments, we consider a deterministic, finite state grid-world domain, the continuous state cart-pole control problem, and a stochastic, discrete state supply-chain problem. The third domain was chosen to show how our algorithm performs when the action space is large and the problem is stochastic. We considered two different reinforcement learning algorithms: on-policy tabular expected SARSA \citep{sutton2018reinforcement}, and off-policy DQN with experience replay \citep{mnih2015human} \footnote{As noted in the previous section, we only need to replace line 11 of Algorithm \ref{alg:bmc} with the DQN update.}. The parameter settings are listed in Tables 1 and 2 in the supplementary materials \footnote{The code and supplementary materials can be found at \url{https://github.com/mike-gimelfarb/bayesian-epsilon-greedy}.}. All experiments were run independently 100 times, and mean curves with shaded standard error are reported.

In the empirical evaluation of Algorithm \ref{alg:bmc}, our goal is to quantify the added value of adapting the $\varepsilon$ parameter in epsilon-greedy policies using a Bayesian approach, rather than compare the performance of epsilon-greedy policies against other approaches, which has been investigated in the literature in various settings \citep{vermorel2005multi,tijsma2016comparing}. Therefore, Algorithm \ref{alg:bmc} is compared against different annealing schedules for $\varepsilon_t$, broken down into the following categories: 
\begin{itemize}
    \item {\bf constant:} $\varepsilon_t = c$, where $c \in \lbrace 0.01, 0.05, 0.1, 0.25, 0.5\rbrace$;
    \item {\bf geometric:} $\varepsilon_t = \frac{1}{2} \rho^t$, where $\rho \in \lbrace 0.85, 0.9, 0.95, 0.975, 0.99\rbrace$ and $t$ is the episode number;
    \item {\bf power:} $\varepsilon_t = \frac{1}{2}(t + 1)^{-\beta}$, where $\beta \in \lbrace 0.25, 0.5, 1.0, 1.5\rbrace$ and $t$ is the episode number;
    \item {\bf adaptive:} VDBE \citep{tokic2010adaptive} with $\varepsilon_0= 0.5$, $\delta = 1 / |\mathcal{A}|$, and $\sigma \in \lbrace 0.01, 0.05, 0.1, 0.5, 1.0, 10.0, 100.0\rbrace$.
\end{itemize}

We do not compare to \cite{tokic2011value}, since that paper falls outside the scope of epsilon-greedy policies. However, we reiterate that it is a trivial matter to interchange VDBE and $\varepsilon$-\texttt{BMC} in that framework.

\subsection{GRID-WORLD}
\label{subsec:gridworld}

The first benchmark problem is the discrete deterministic 5-by-5 grid-world navigation problem with sub-goals presented in \cite{ng1999policy}. Valid moves incur a cost of 0.1, and invalid moves incur a cost of 0.2, in order to encourage the agent to solve the task in the least amount of time. We set $\gamma = 0.99$. Testing consists of running a single episode, starting from the same initial state, using the greedy policy at the end of each episode. The results are shown in Figures \ref{fig:gridworld_sarsa} and \ref{fig:gridworld_dqn}.

\begin{figure*}[t!]
    \centering
    \includegraphics[width=0.33\linewidth]{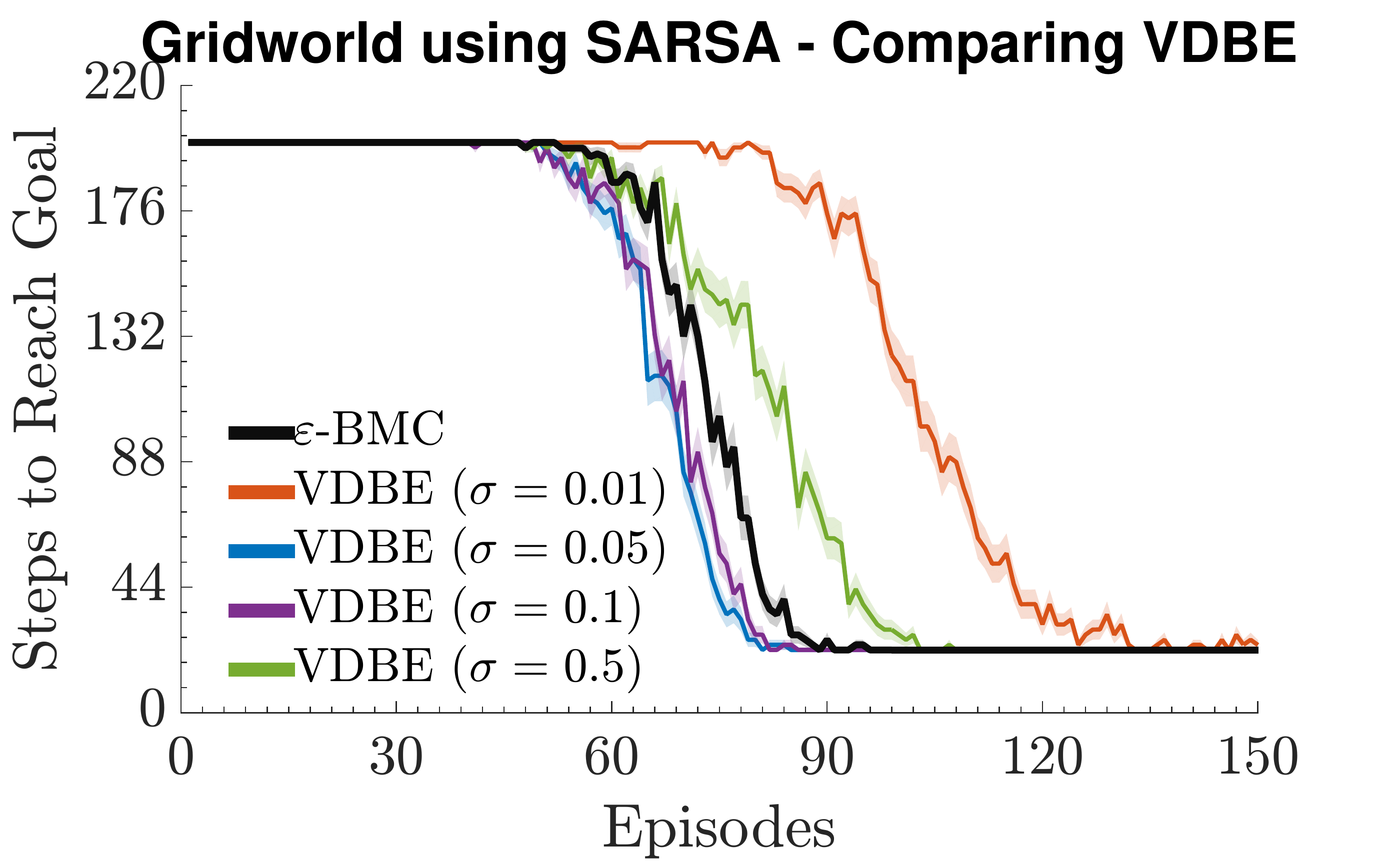}
    \includegraphics[width=0.33\linewidth]{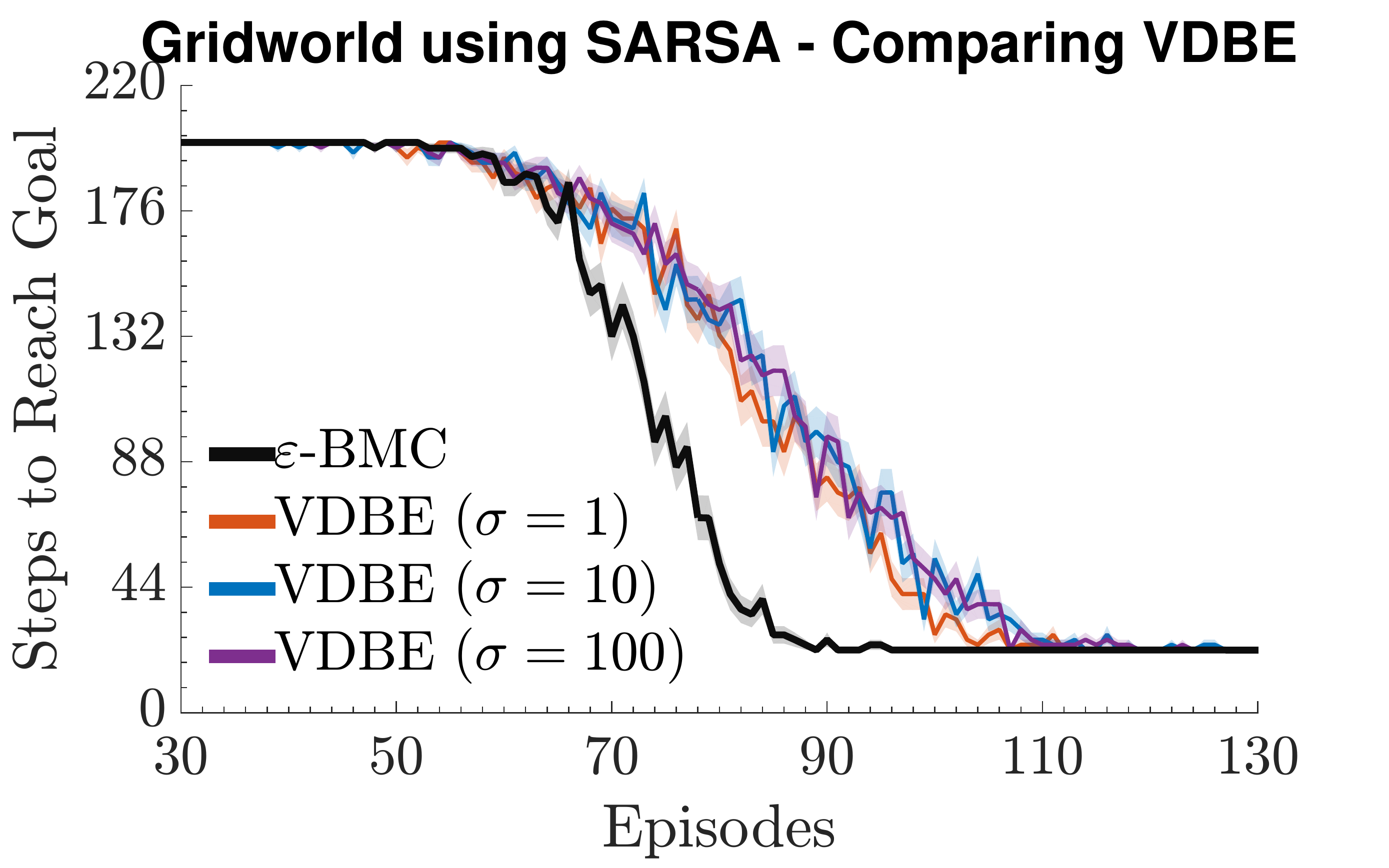}
    \includegraphics[width=0.33\linewidth]{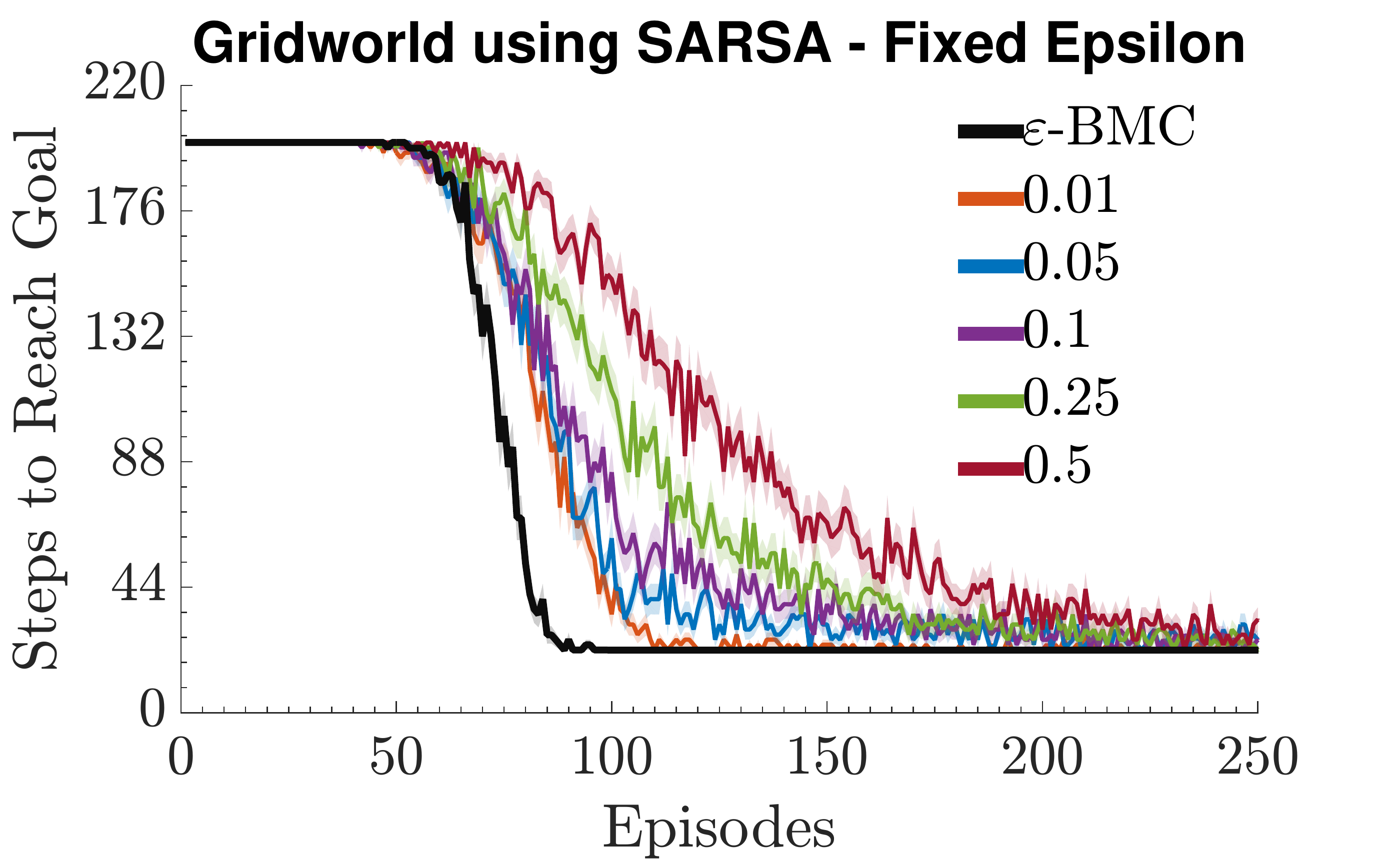} 
    \includegraphics[width=0.33\linewidth]{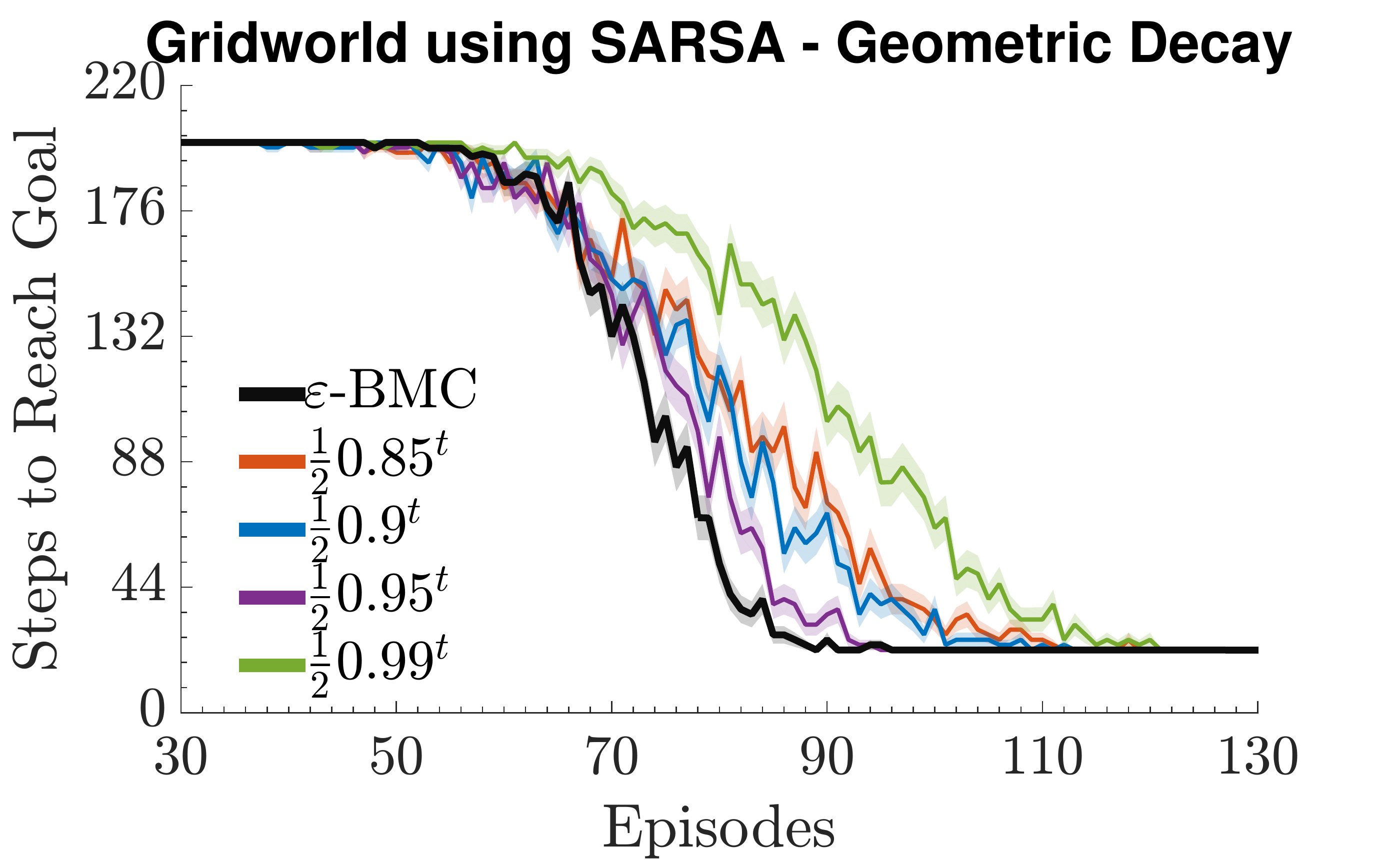}
    \includegraphics[width=0.33\linewidth]{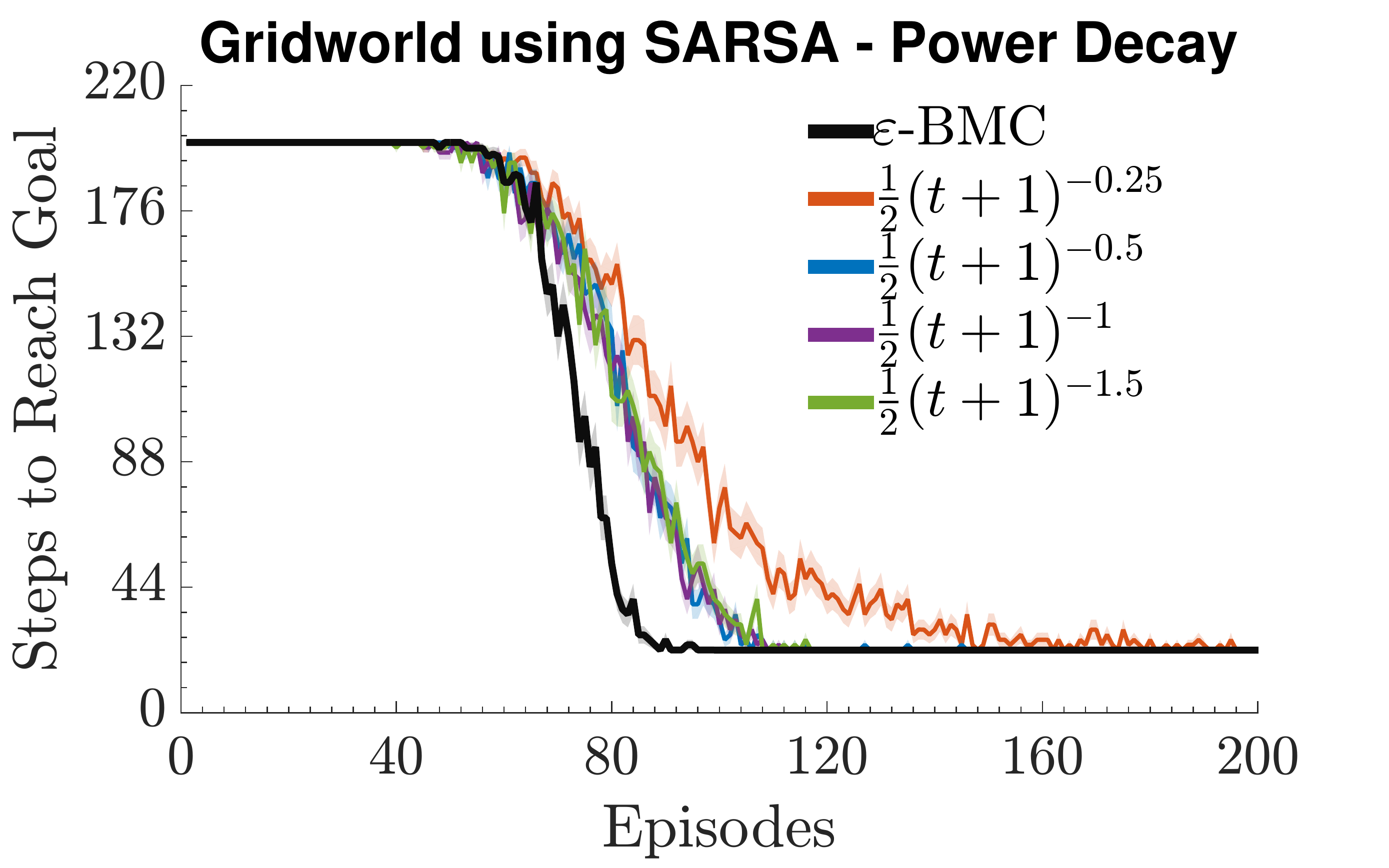} 
    \includegraphics[width=0.33\linewidth]{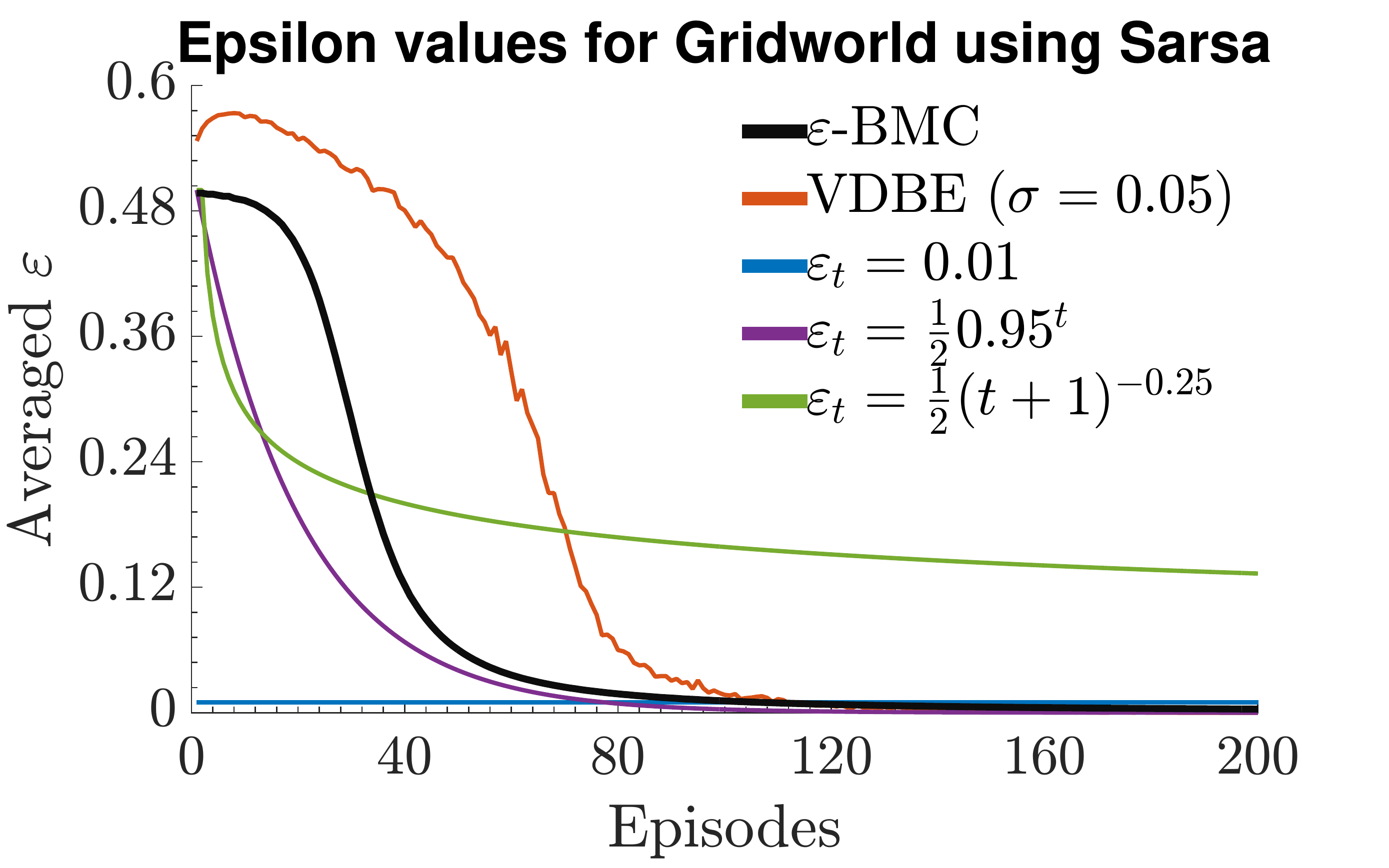}
    \caption{Average performance (steps to reach the final goal) on the grid-world domain using expected SARSA.}
\label{fig:gridworld_sarsa}
\end{figure*}

\begin{figure*}[t!]
    \centering
    \includegraphics[width=0.33\linewidth]{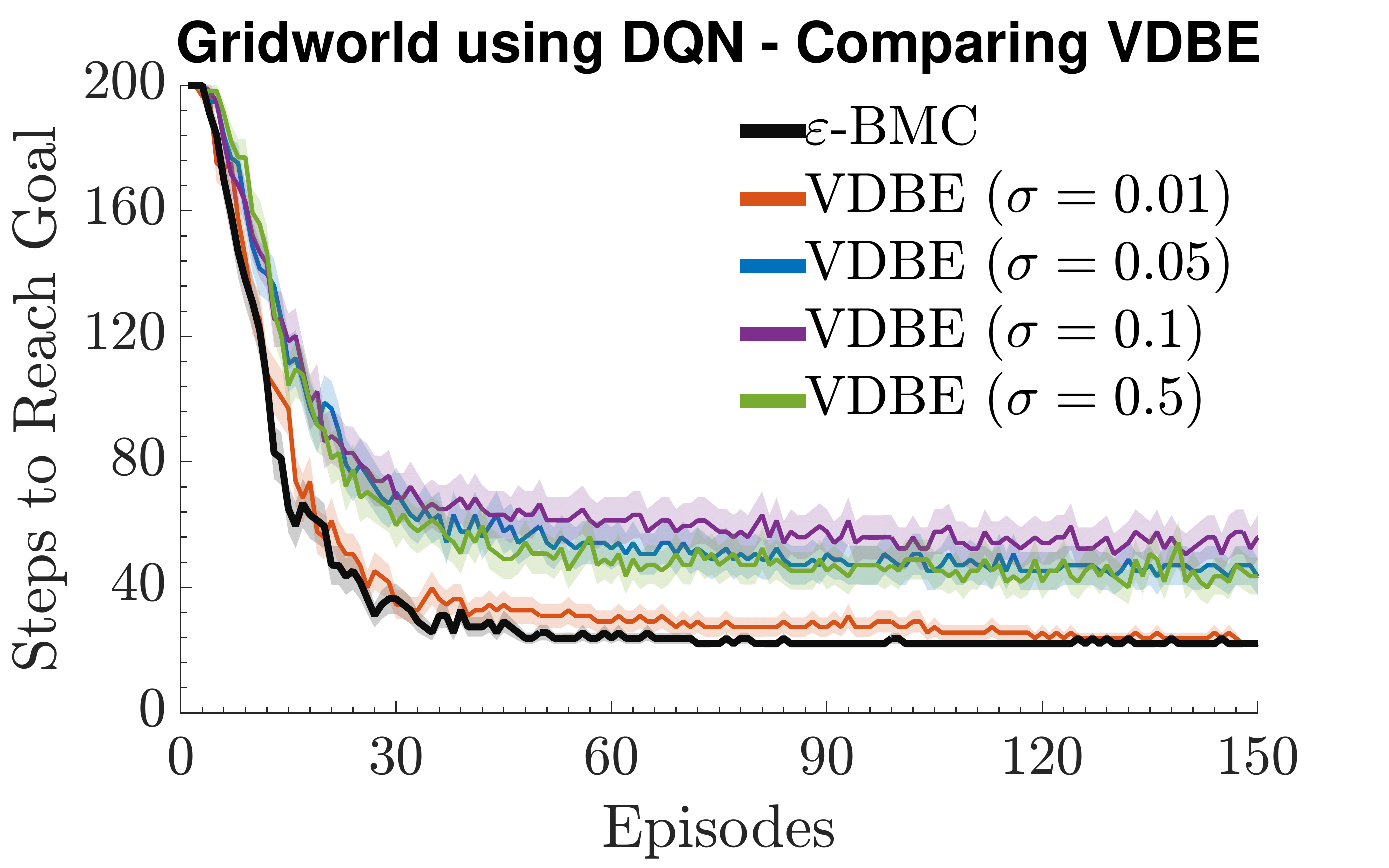}
    \includegraphics[width=0.33\linewidth]{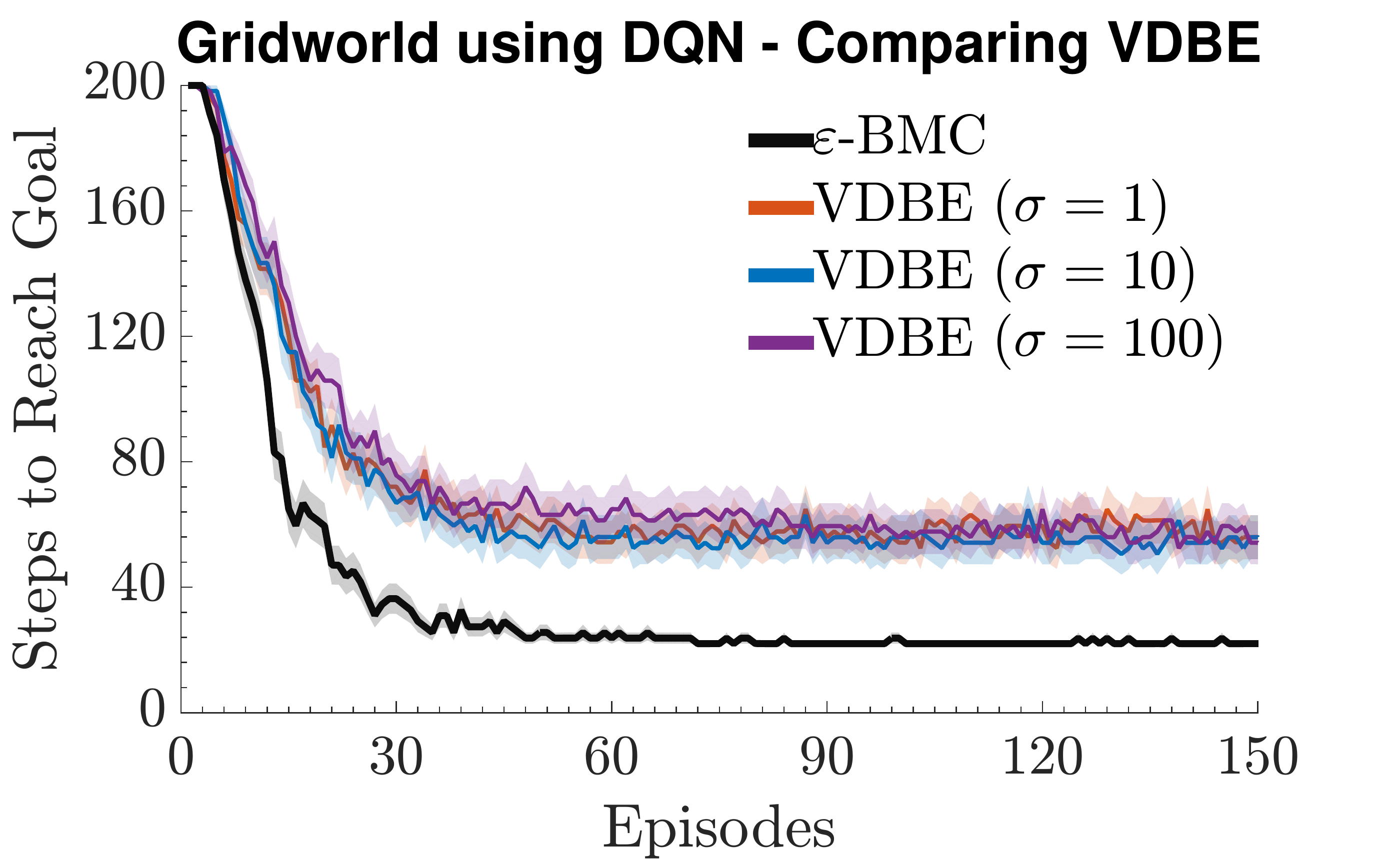}
    \includegraphics[width=0.33\linewidth]{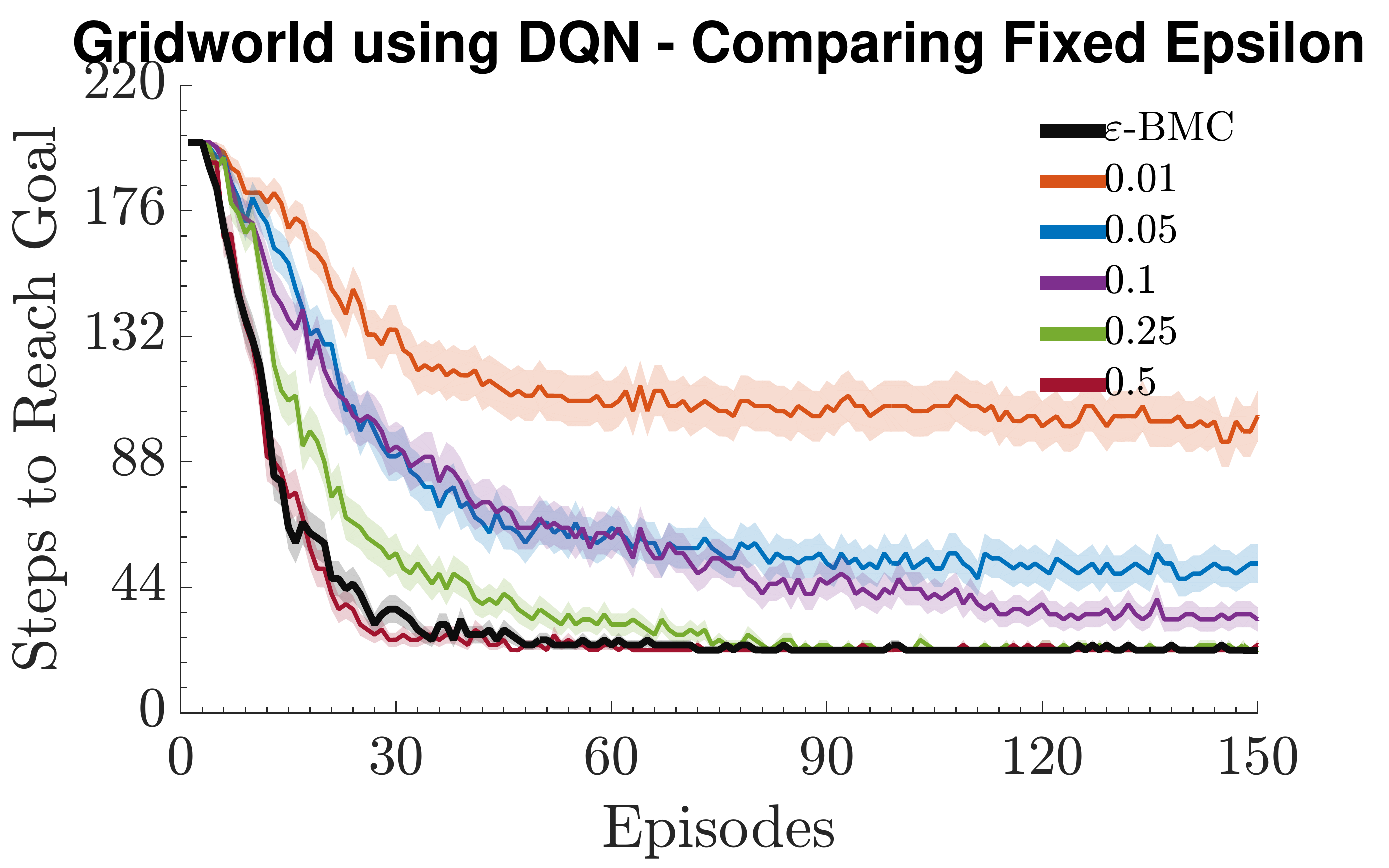} 
    \includegraphics[width=0.33\linewidth]{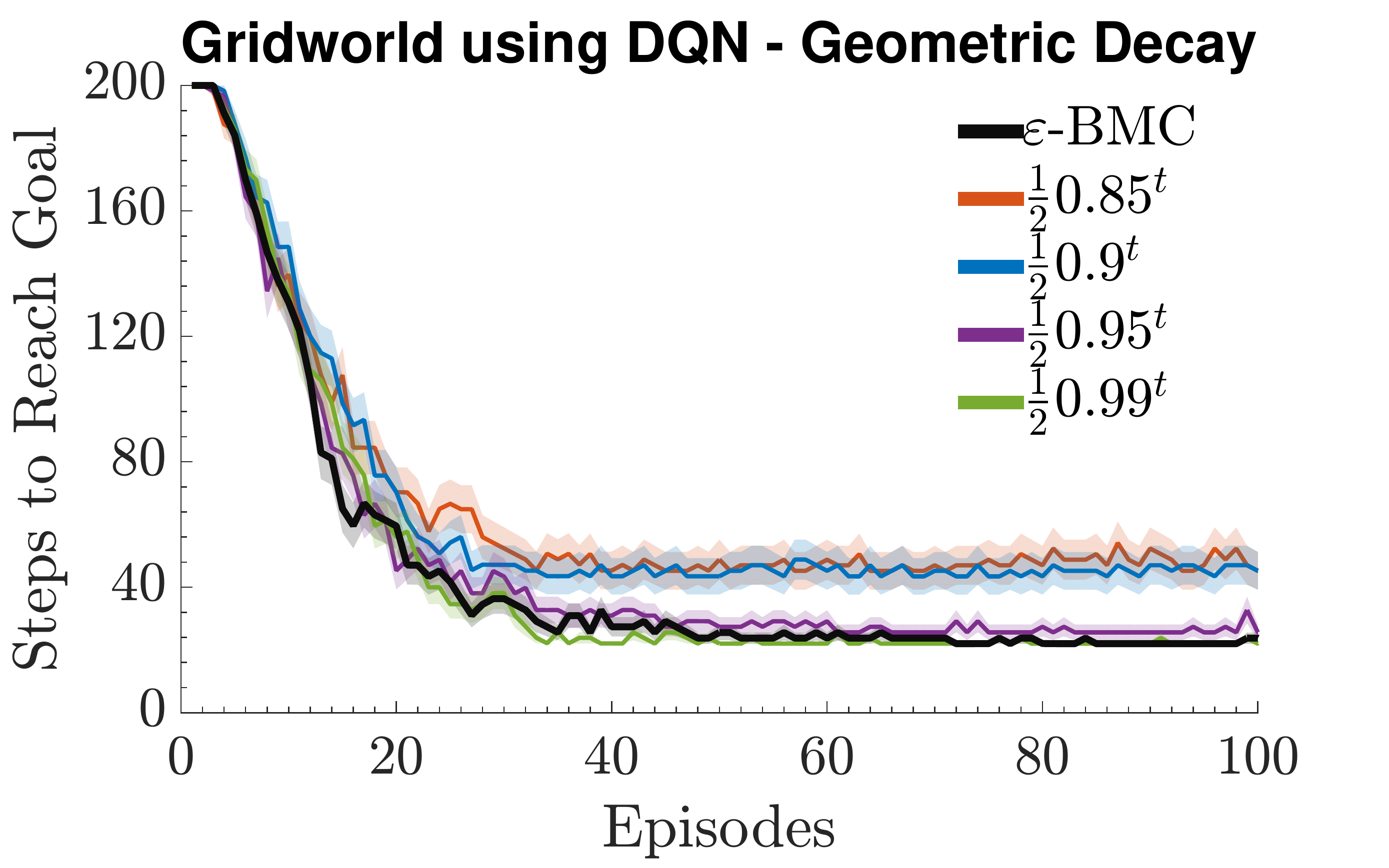}
    \includegraphics[width=0.33\linewidth]{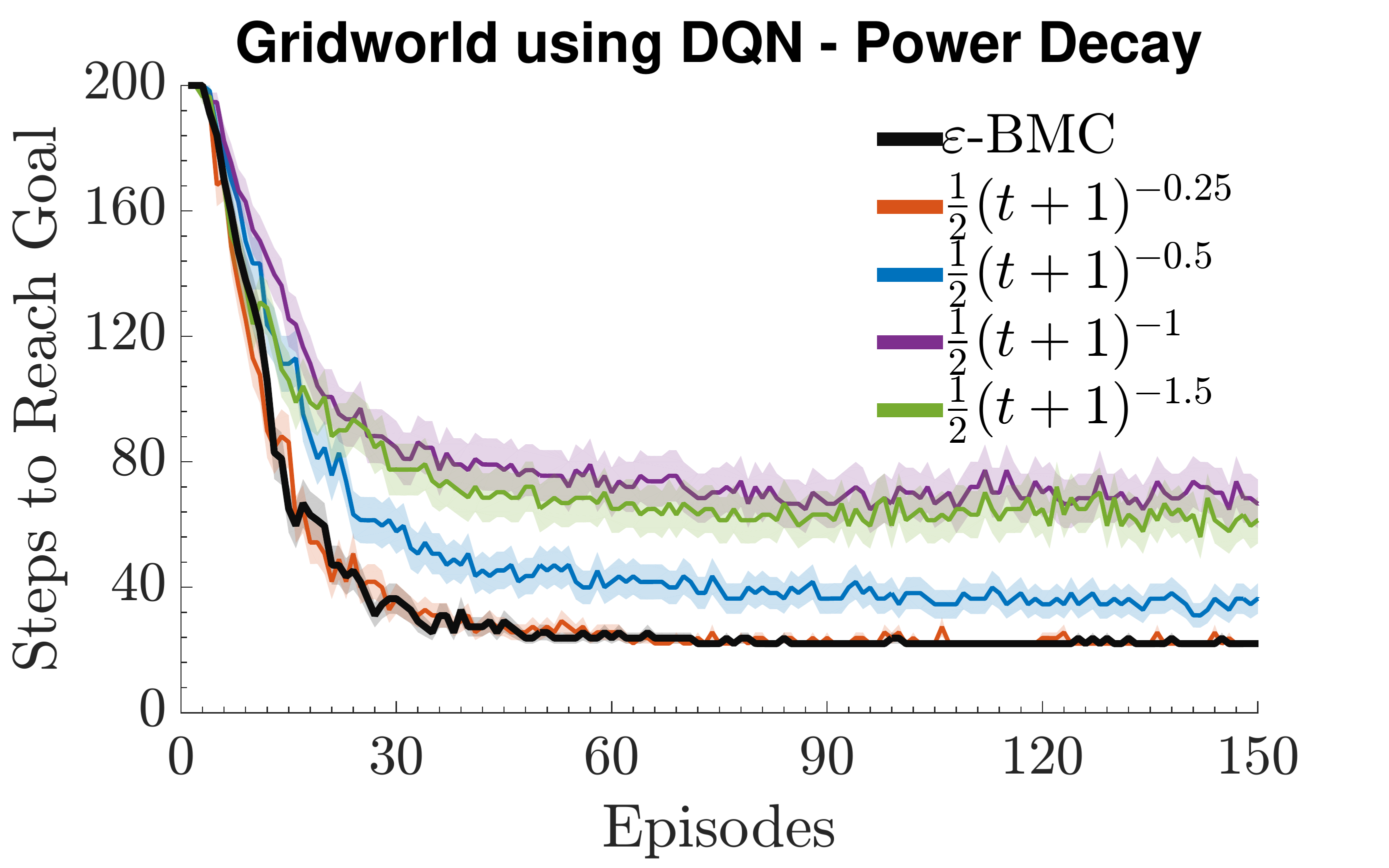} 
    \includegraphics[width=0.33\linewidth]{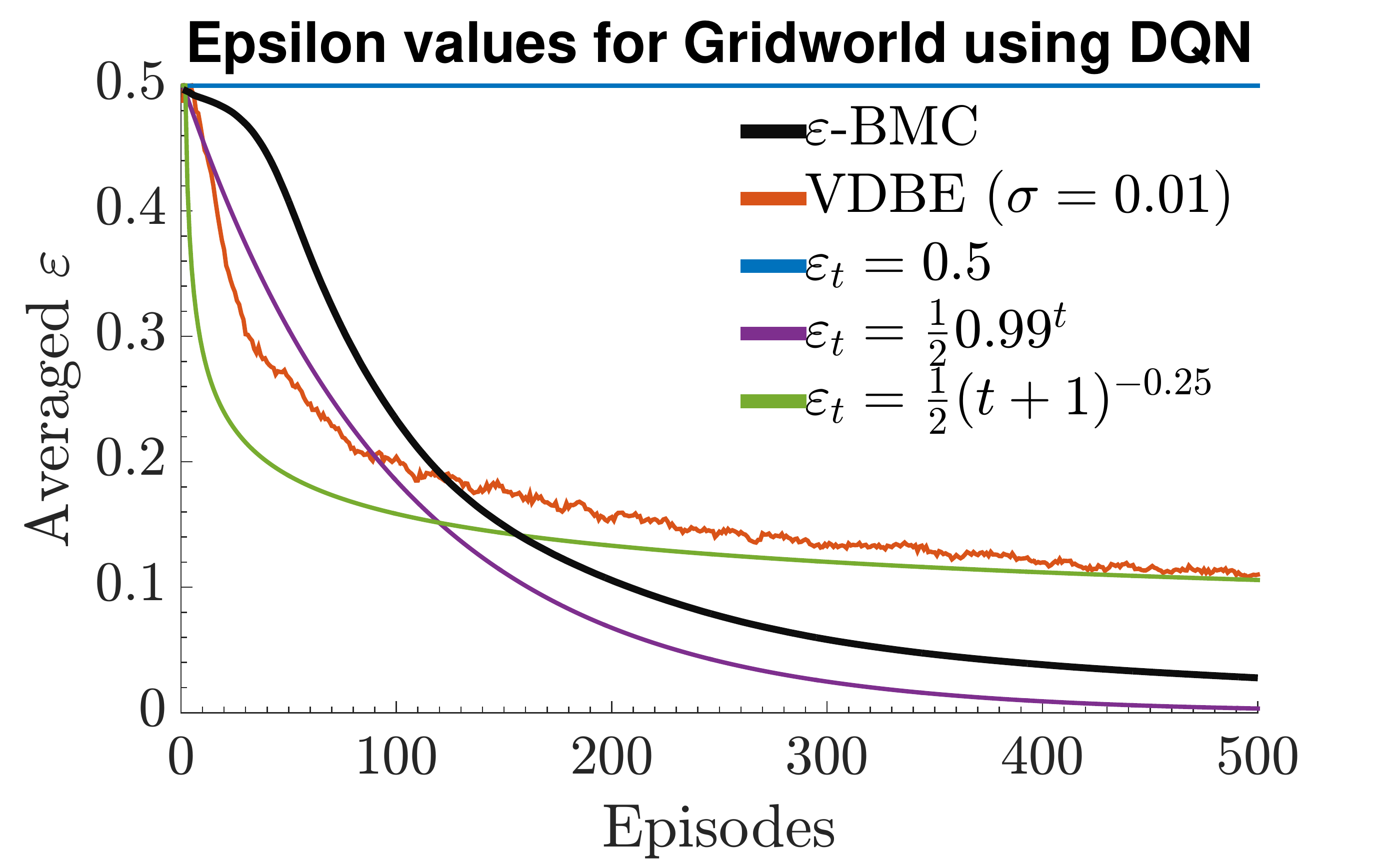} 
    \caption{Average performance (steps to reach the final goal) on the grid-world domain using deep Q learning.}
\label{fig:gridworld_dqn}
\end{figure*}

\subsection{CART-POLE}
\label{subsec:cartpole}

The second problem is the continuous deterministic cart-pole control problem. A reward of 1.0 is provided until the pole falls, to encourage the agent to keep the pole upright. We also set $\gamma = 0.95$. To run the tabular expected SARSA algorithm and VDBE, we discretize the four-dimensional state space into $3\times 3\times 4\times 3 = 108$ equal regions. Since the initial position is randomized, testing consists of evaluating the greedy policy on 10 independent episodes and averaging the returns. Since over-fitting was a significant concern for DQN, we stop training as soon as perfect test performance (the pole has not been dropped) was observed over four consecutive episodes. The results are shown in Figures \ref{fig:cartpole_sarsa} and \ref{fig:cartpole_dqn}.

\begin{figure*}[t!]
    \centering
    \includegraphics[width=0.33\linewidth]{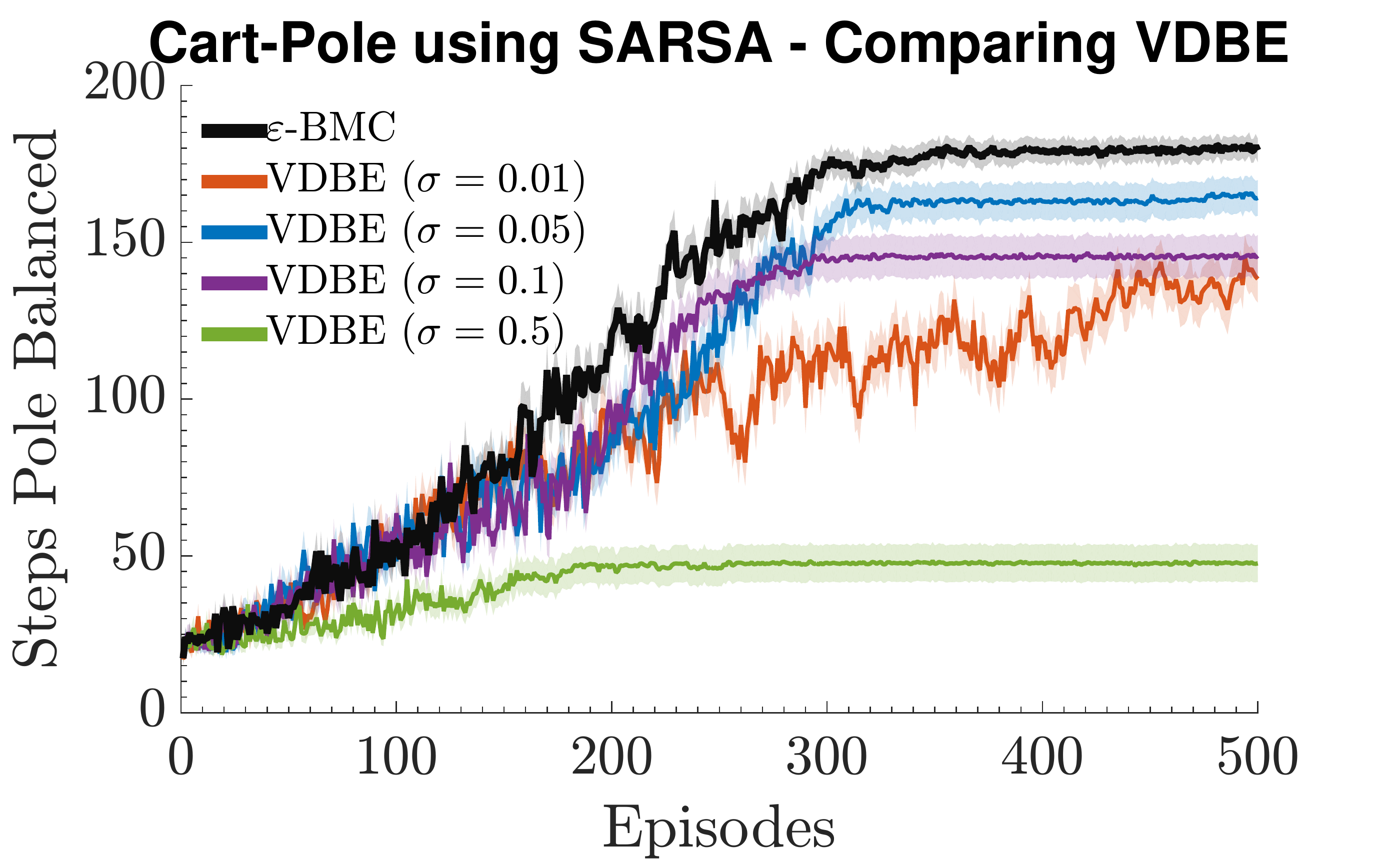}
    \includegraphics[width=0.33\linewidth]{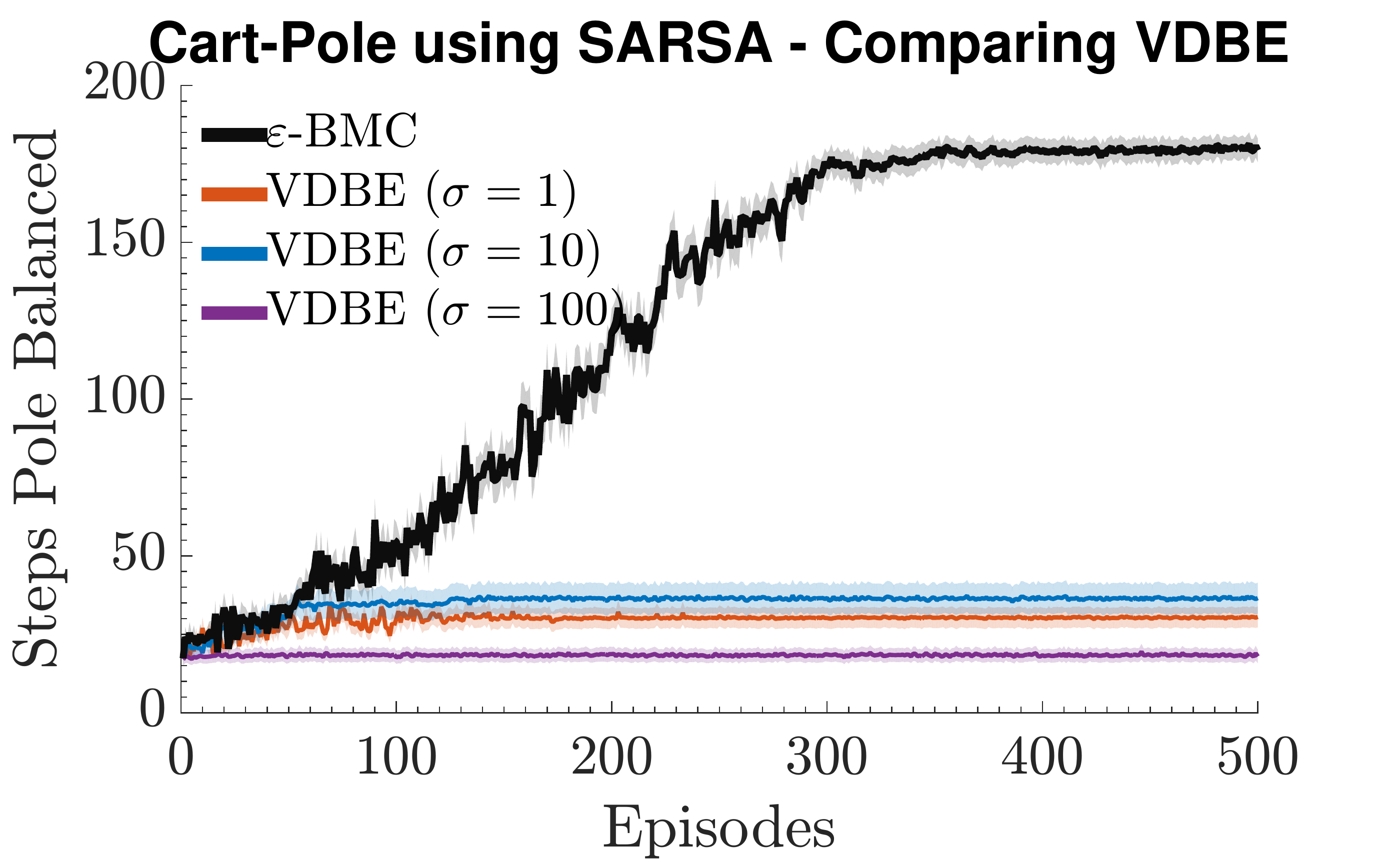}
    \includegraphics[width=0.33\linewidth]{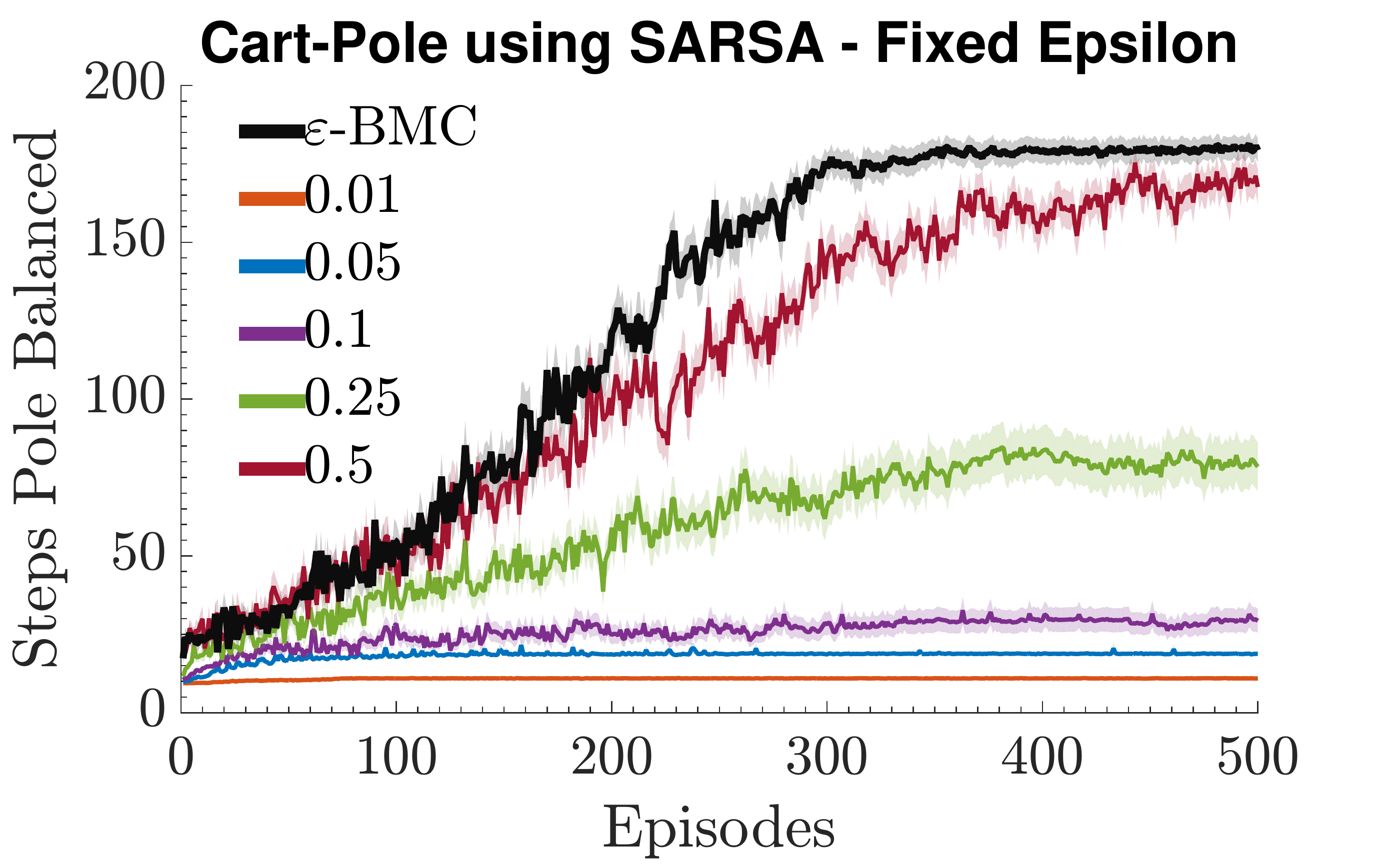} 
    \includegraphics[width=0.33\linewidth]{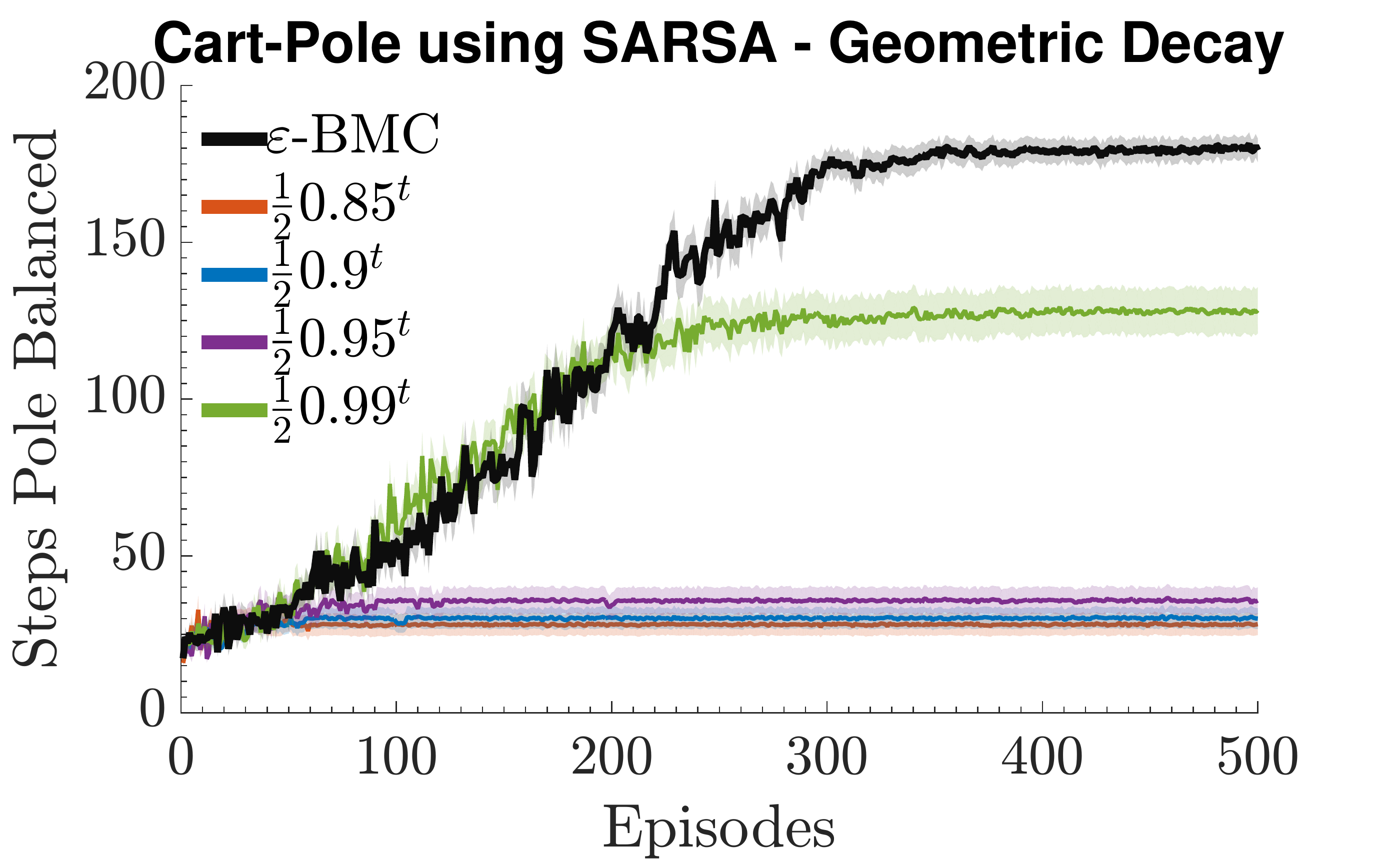}
    \includegraphics[width=0.33\linewidth]{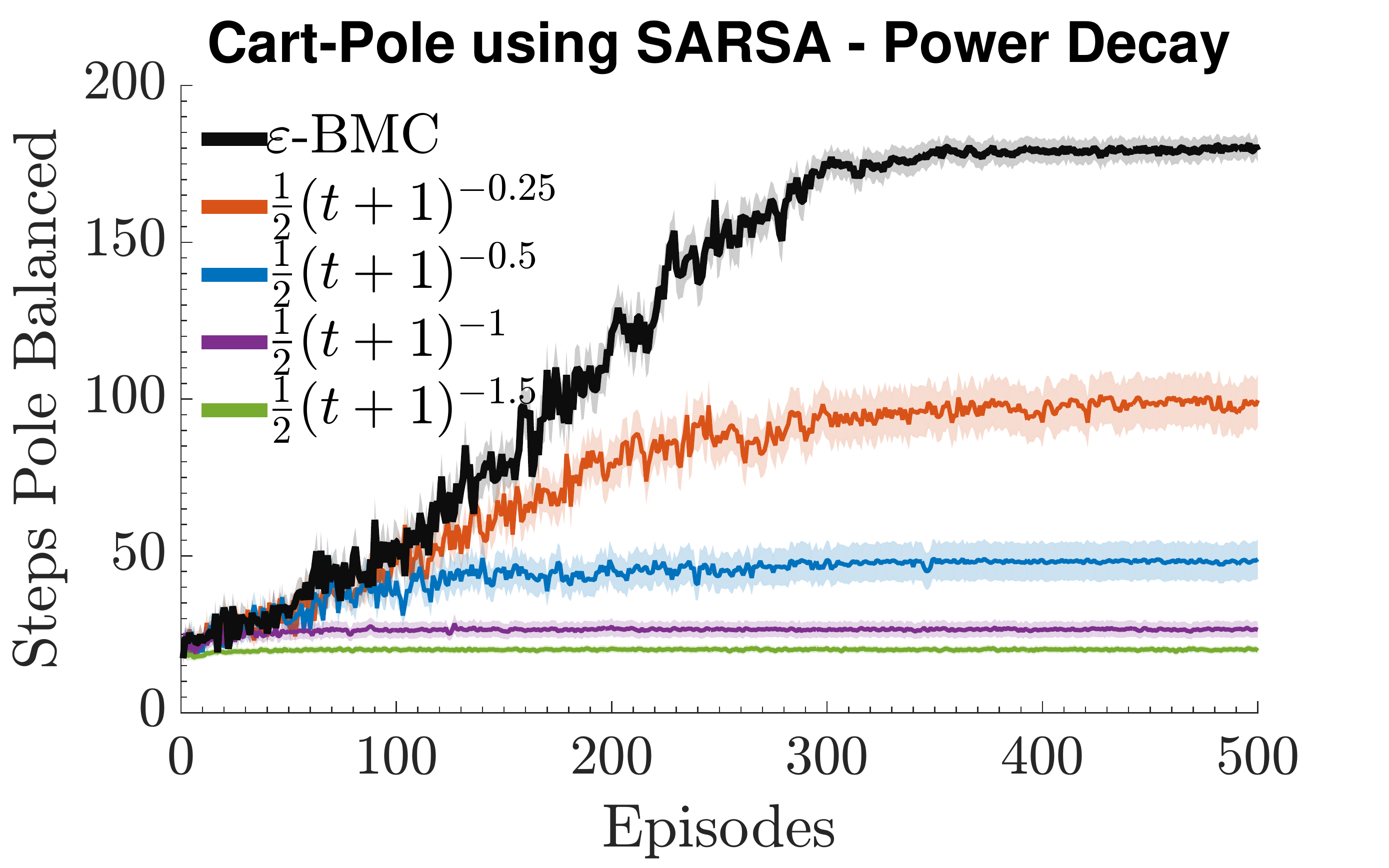} 
    \includegraphics[width=0.33\linewidth]{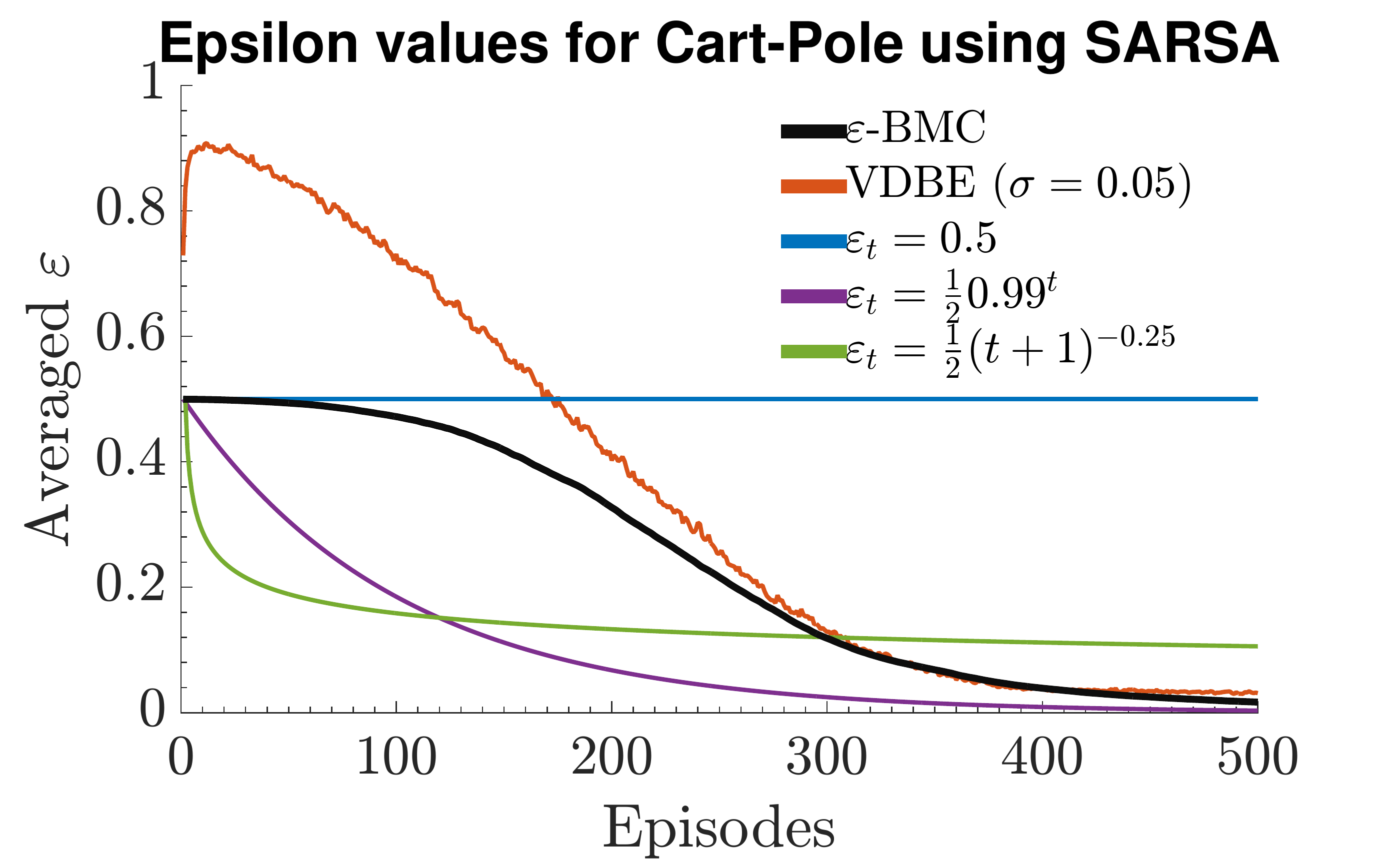} 
    \caption{Average performance (time steps the pole is balanced) on the cart-pole domain using expected SARSA.}
\label{fig:cartpole_sarsa}
\end{figure*}

\begin{figure*}[t!]
    \centering
    \includegraphics[width=0.33\linewidth]{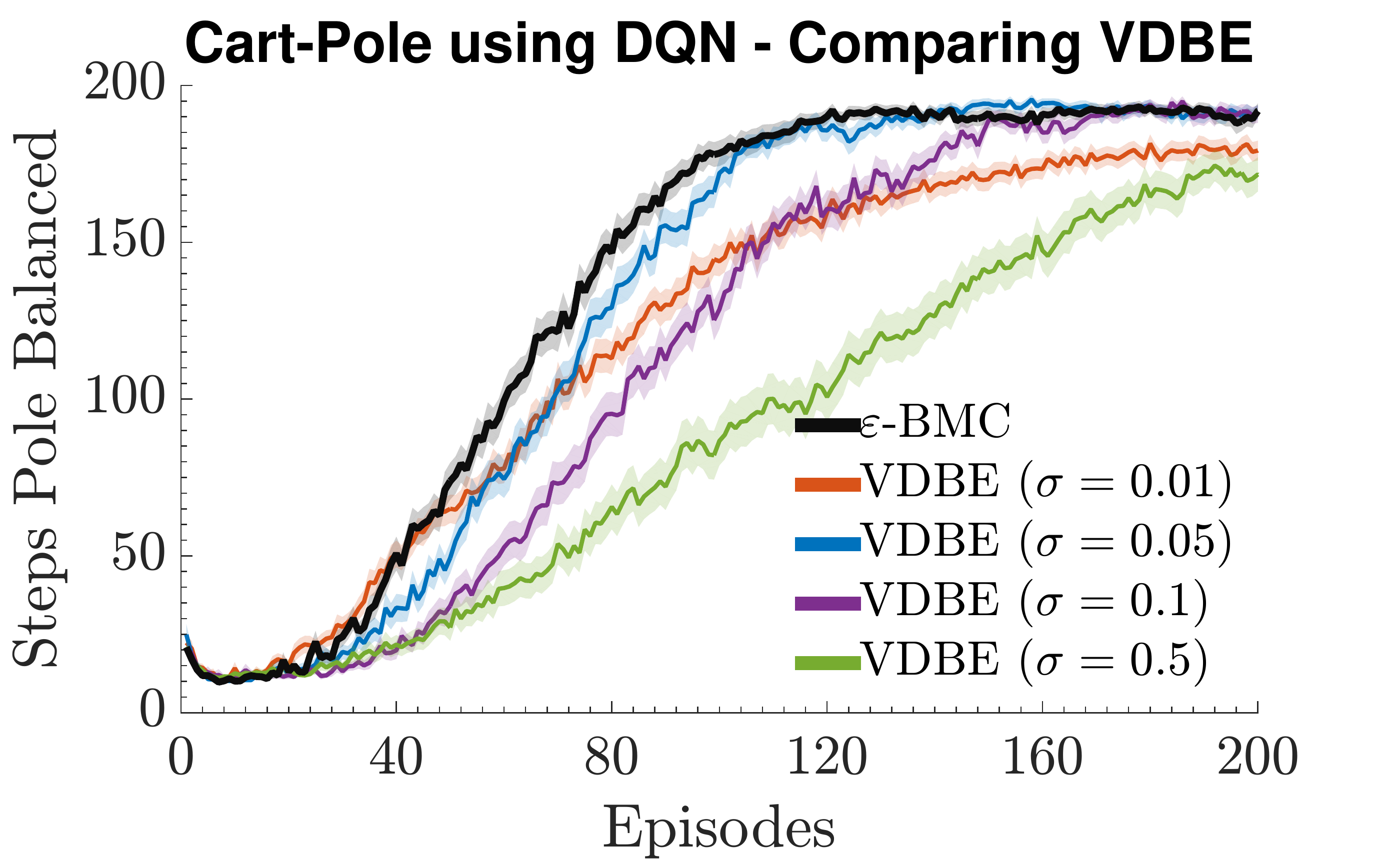}
    \includegraphics[width=0.33\linewidth]{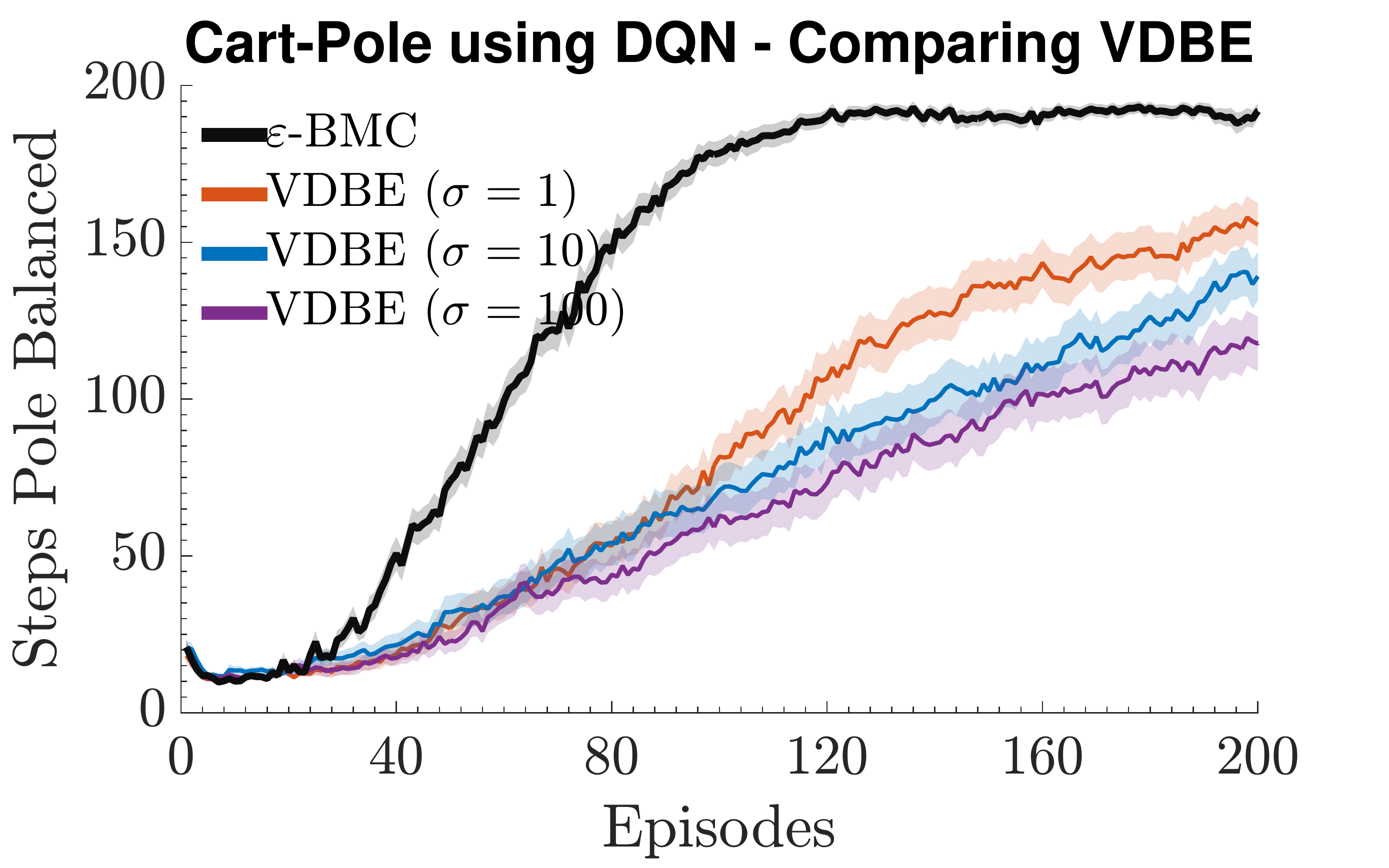}
    \includegraphics[width=0.33\linewidth]{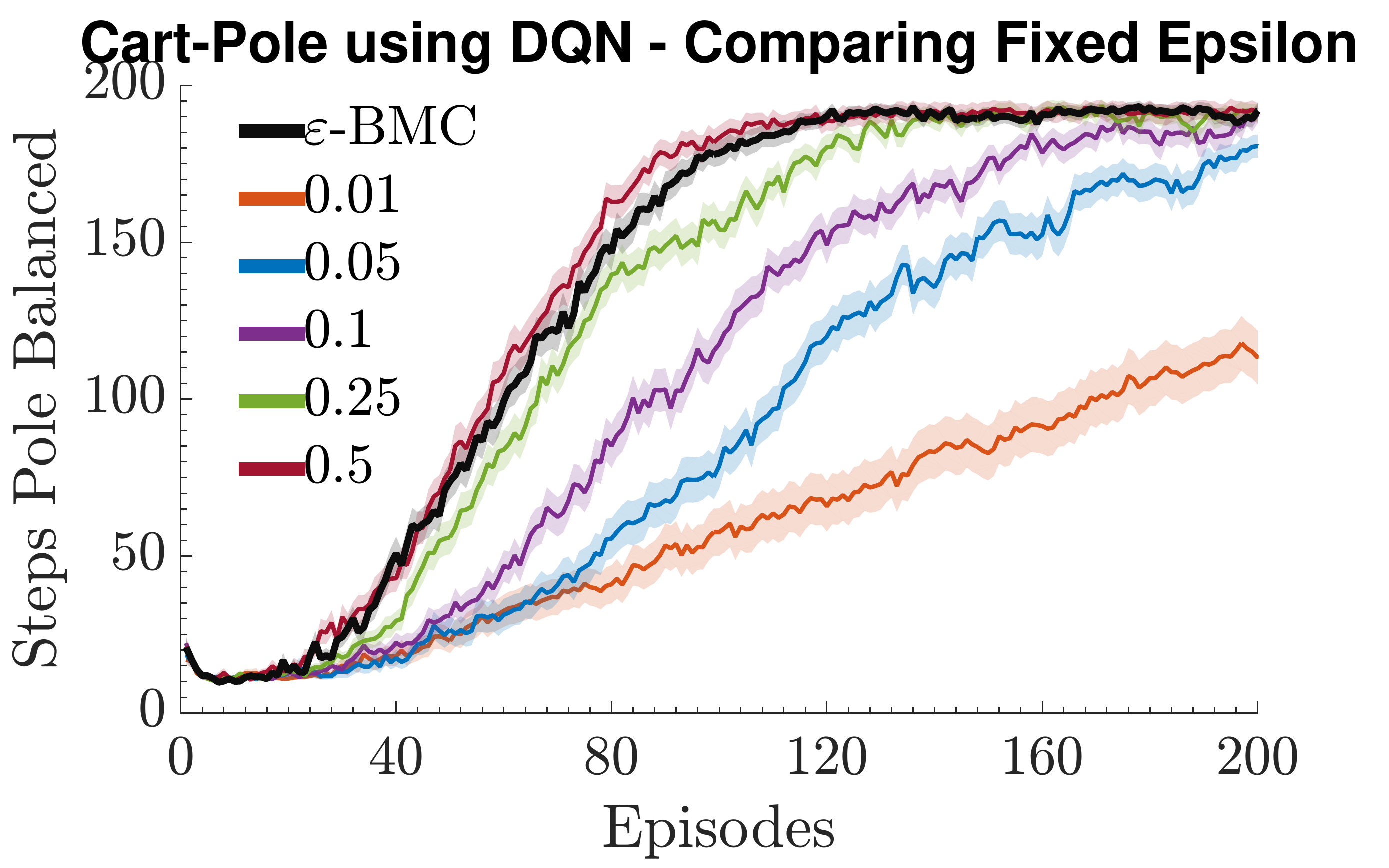} 
    \includegraphics[width=0.33\linewidth]{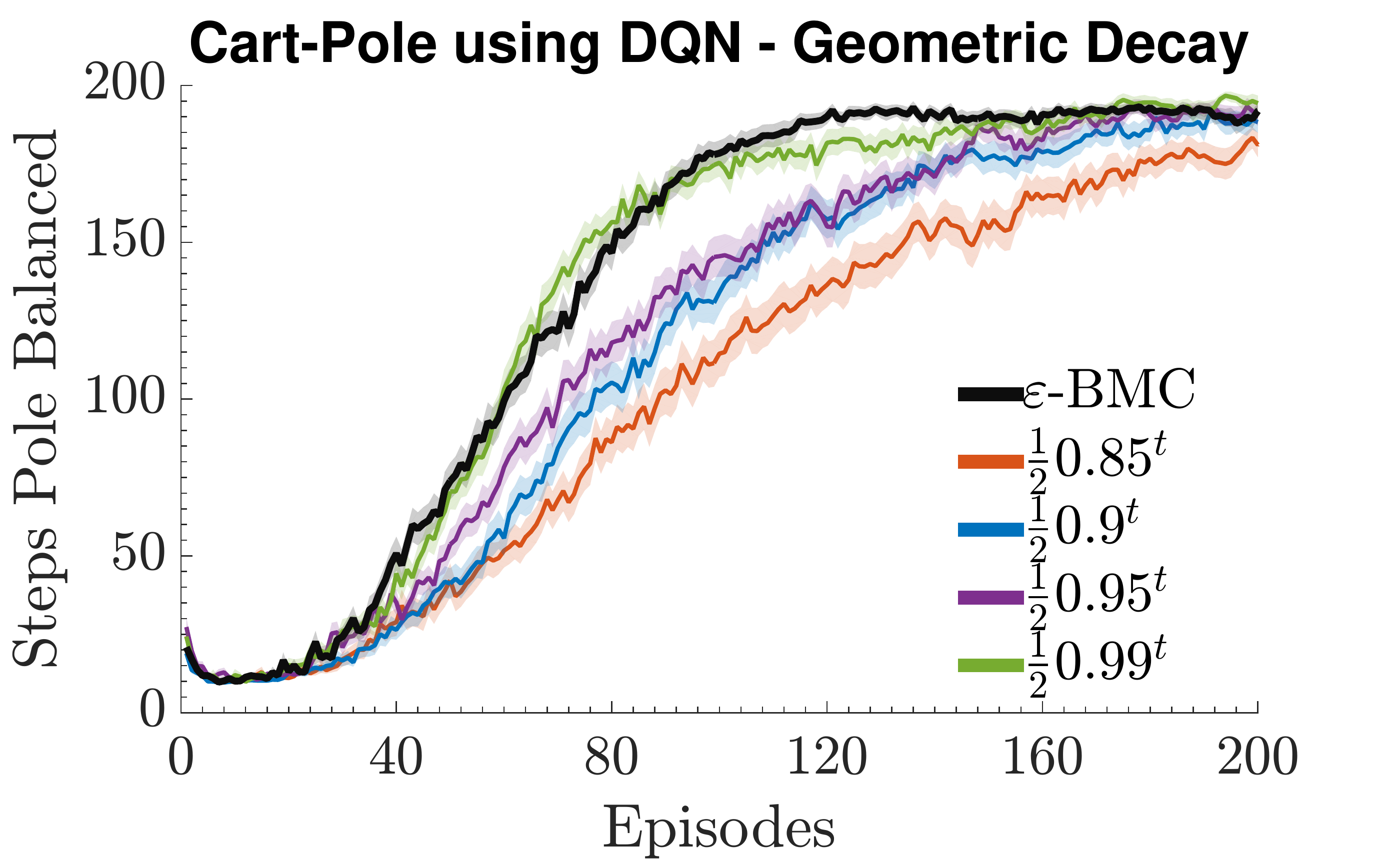}
    \includegraphics[width=0.33\linewidth]{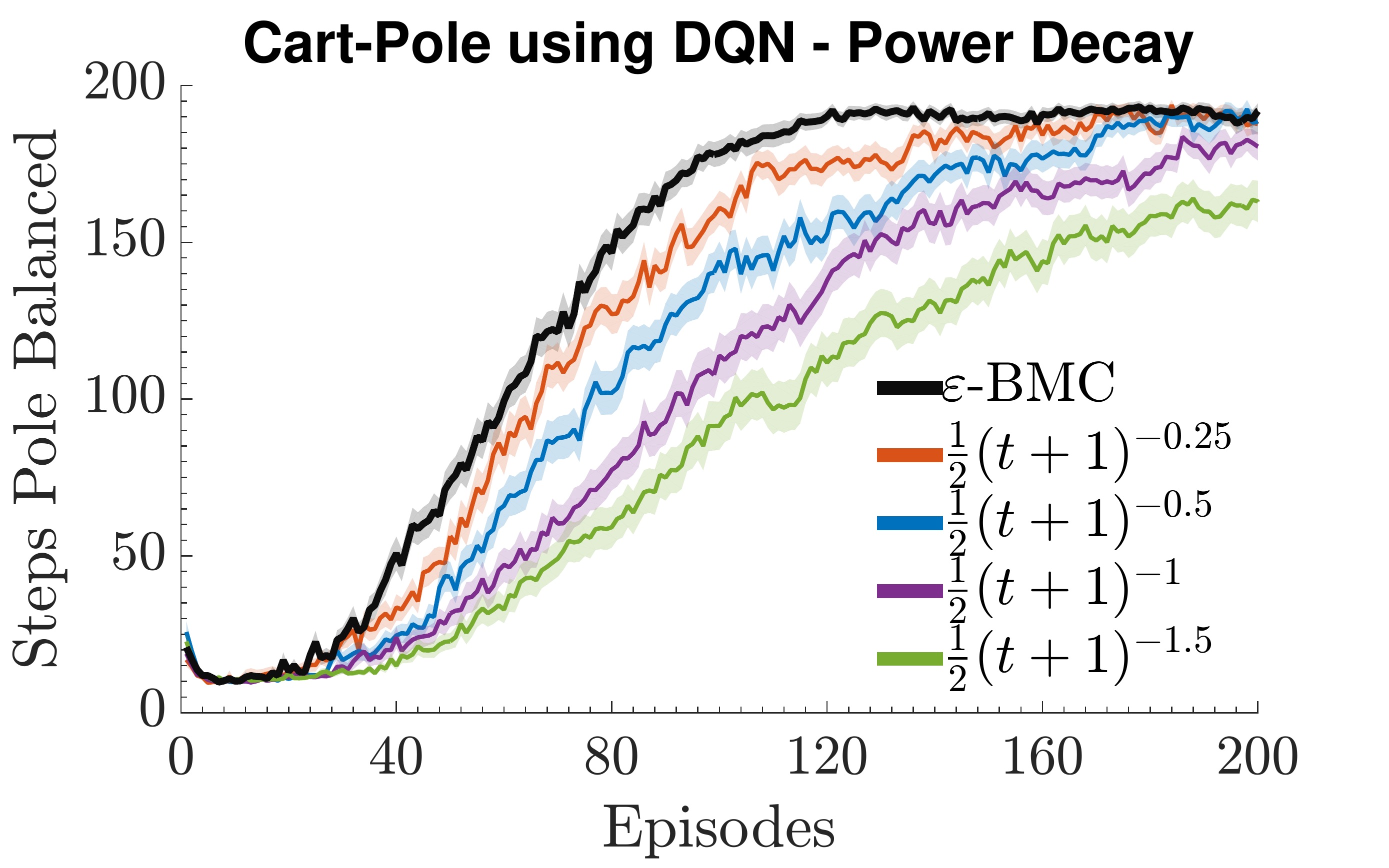} 
    \includegraphics[width=0.33\linewidth]{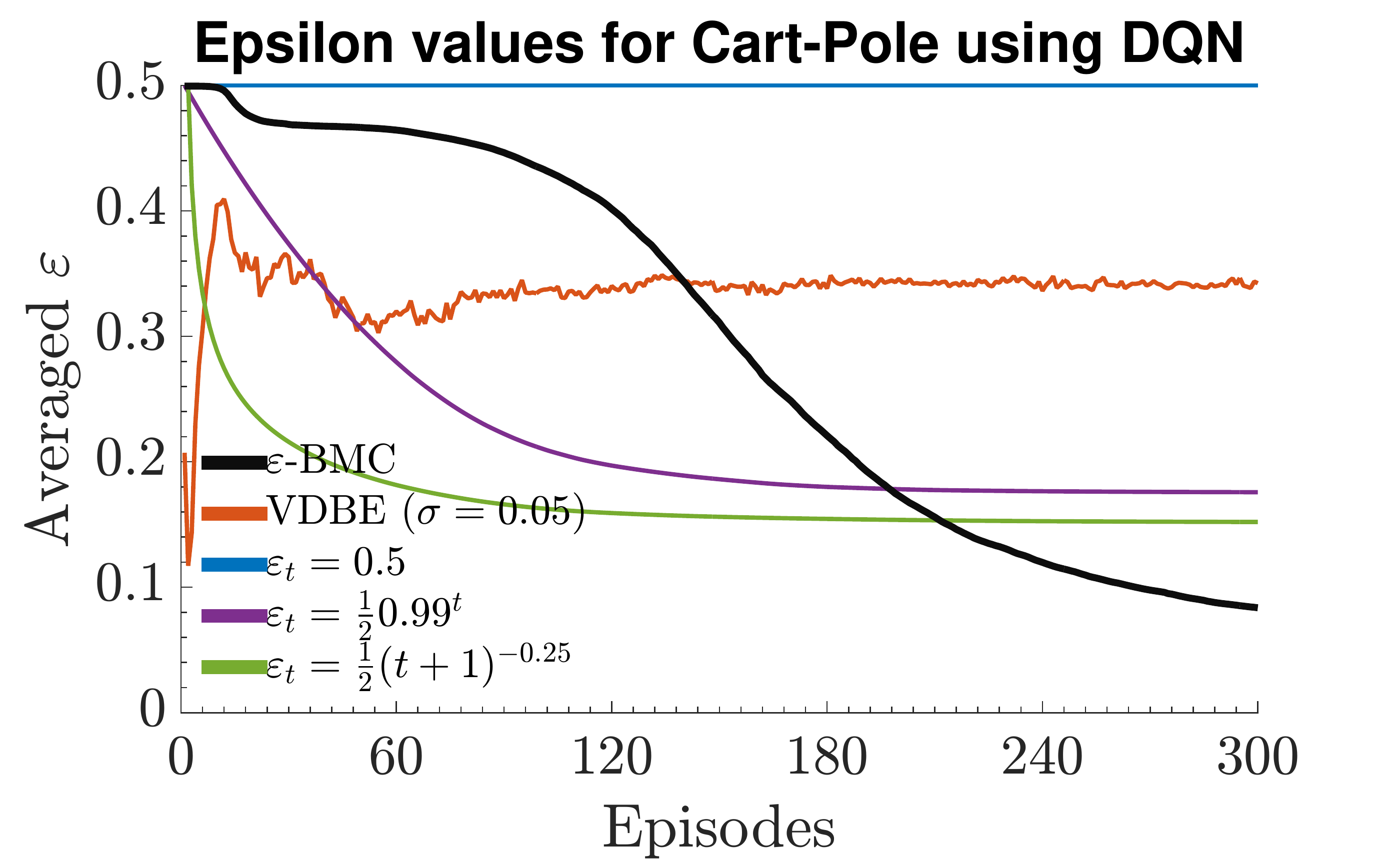} 
    \caption{Average performance (time steps the pole is balanced) on the cart-pole domain using deep Q-learning.}
\label{fig:cartpole_dqn}
\end{figure*}

\subsection{SUPPLY-CHAIN MANAGEMENT}
\label{subsec:inventory}

This supply-chain management problem was described in \cite{kemmer2018supply}, and consists of a factory and warehouses. The agent must decide how much inventory to produce at the factory and how much inventory to ship to the warehouse(s) to meet the demand. The parameters used in our experiment are: $K = 1$, $p = 0.5$, $\kappa_{pr}= 0.1$, $\kappa_{st,j} = 0.02$, $\kappa_{tr,j} = 0.1$, $\zeta_j = 5, c_j = 50, \rho_{max} = 10$, and demand follows a Poisson distribution with rate $\lambda = 2.5$. We also set a transportation limit of $10$ items per period, and assume that unfulfilled demand is not backlogged, but lost forever. The initial state is always $(10, 0)$ for training and testing (e.g. 10 items initially at the factory and 0 at the warehouse). We set $\gamma = 0.95$. Like cart-pole, testing is done by averaging the returns of 10 independent trials using the greedy policy. The results are illustrated in Figures \ref{fig:inventory_sarsa} and \ref{fig:inventory_dqn}.

\begin{figure*}[t!]
    \centering
    \includegraphics[width=0.33\linewidth]{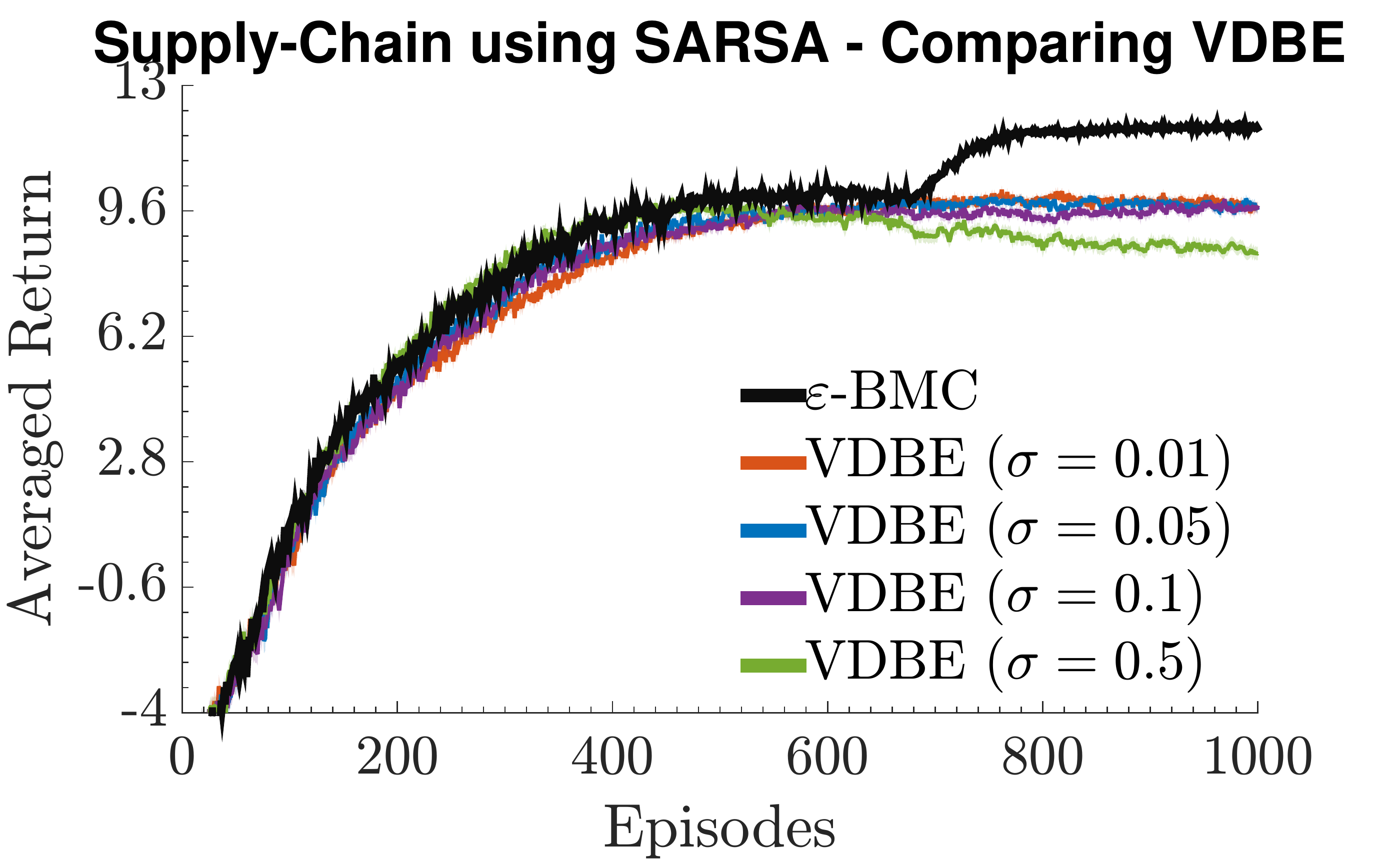}
    \includegraphics[width=0.33\linewidth]{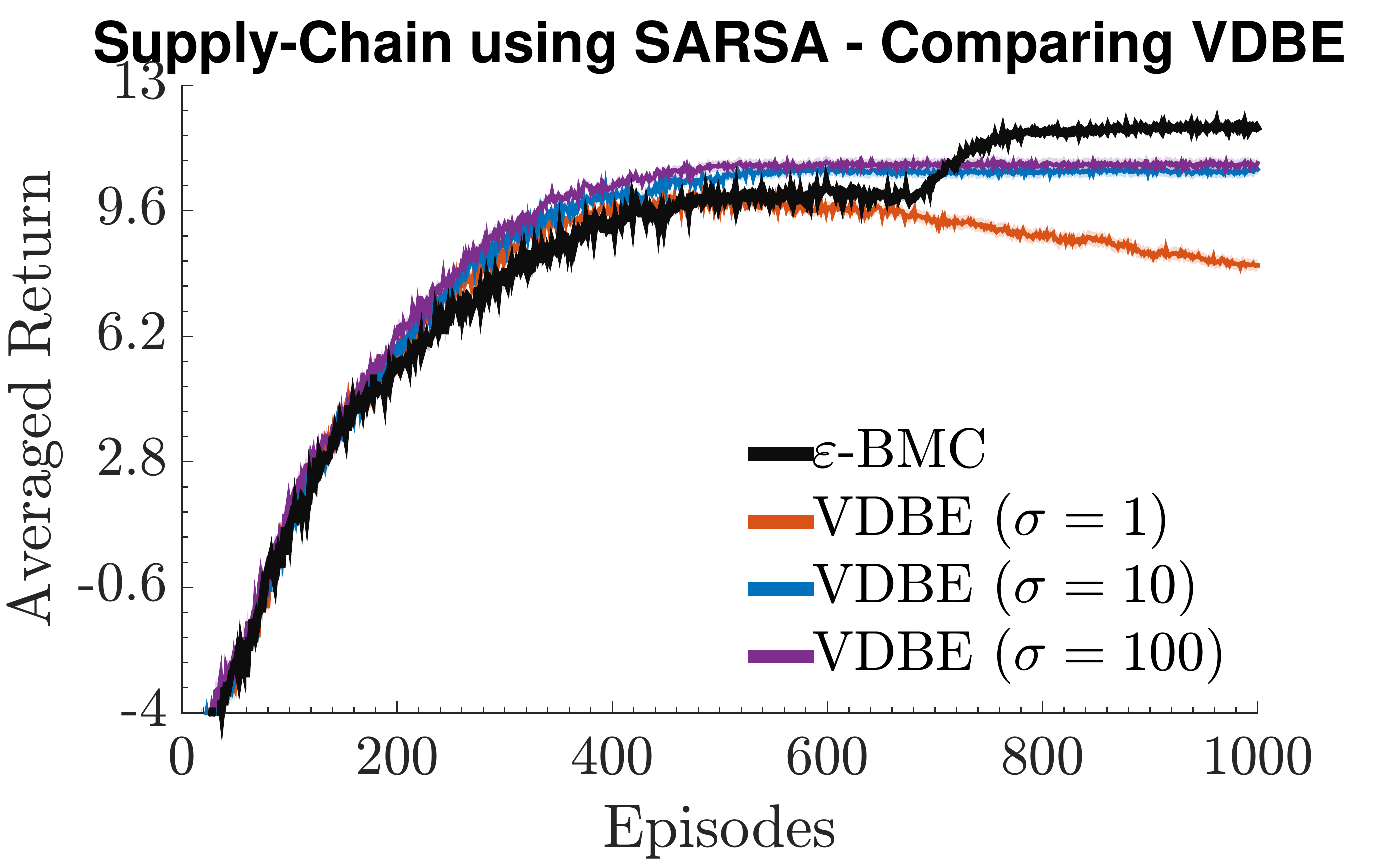}
    \includegraphics[width=0.33\linewidth]{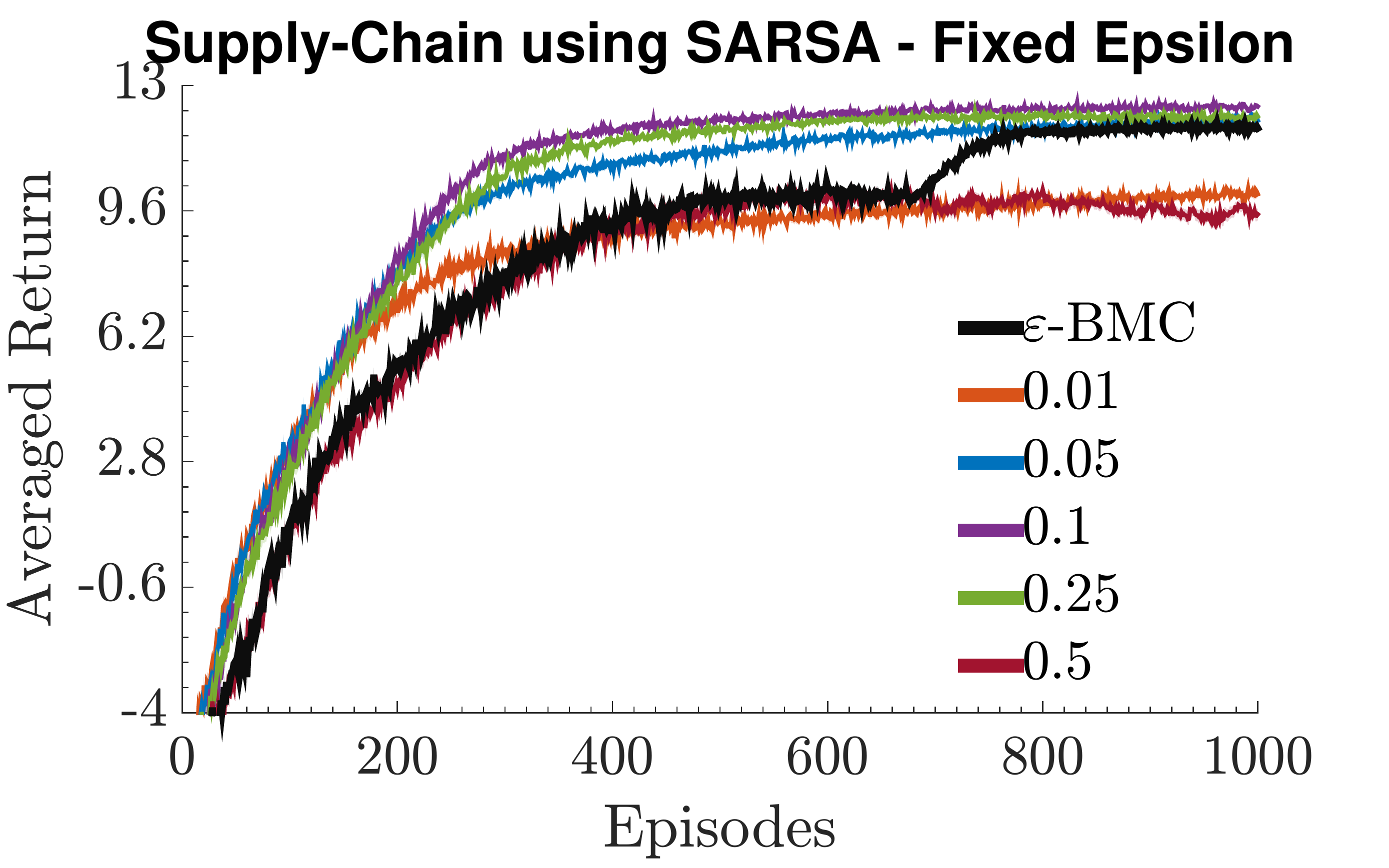} 
    \includegraphics[width=0.33\linewidth]{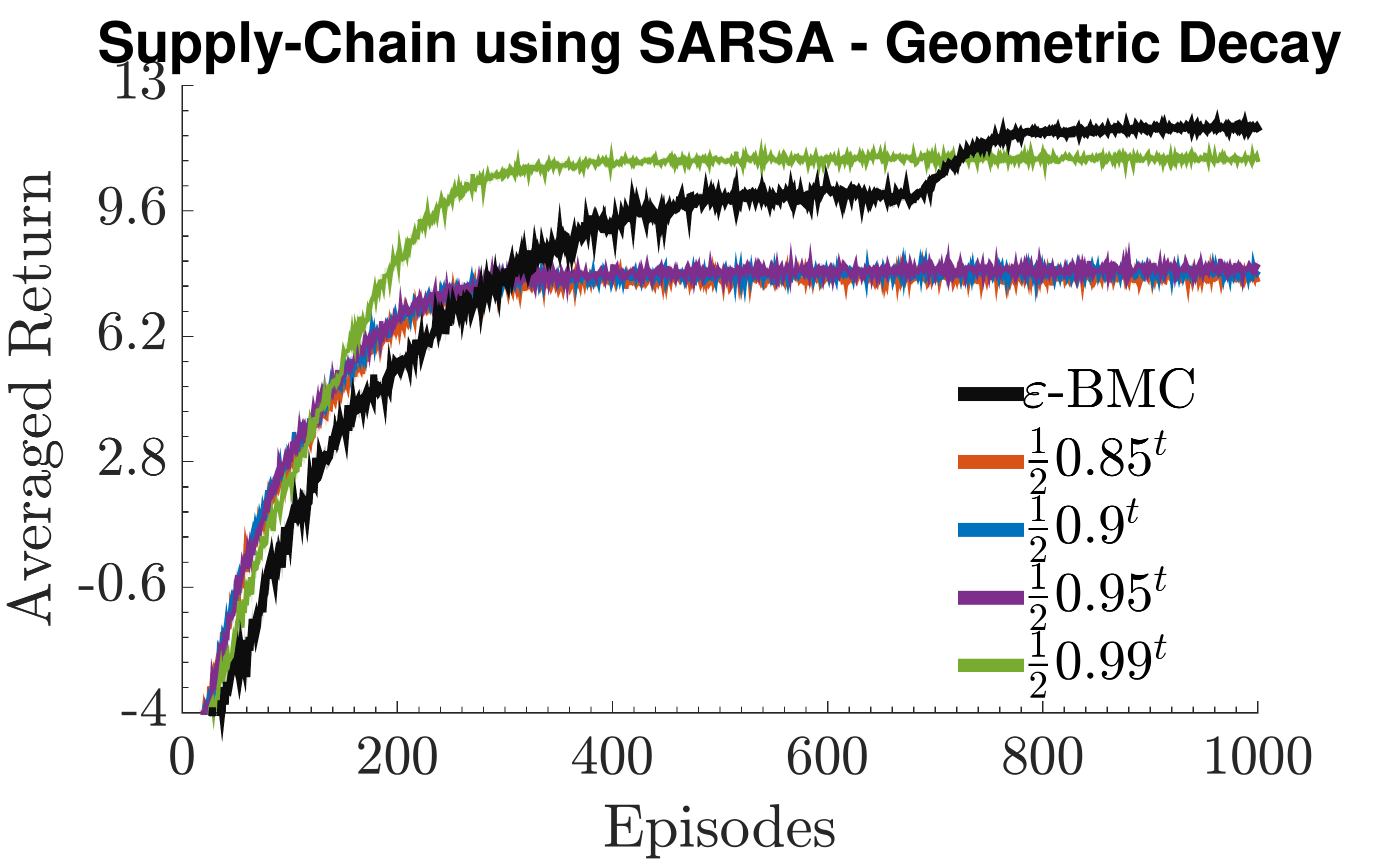}
    \includegraphics[width=0.33\linewidth]{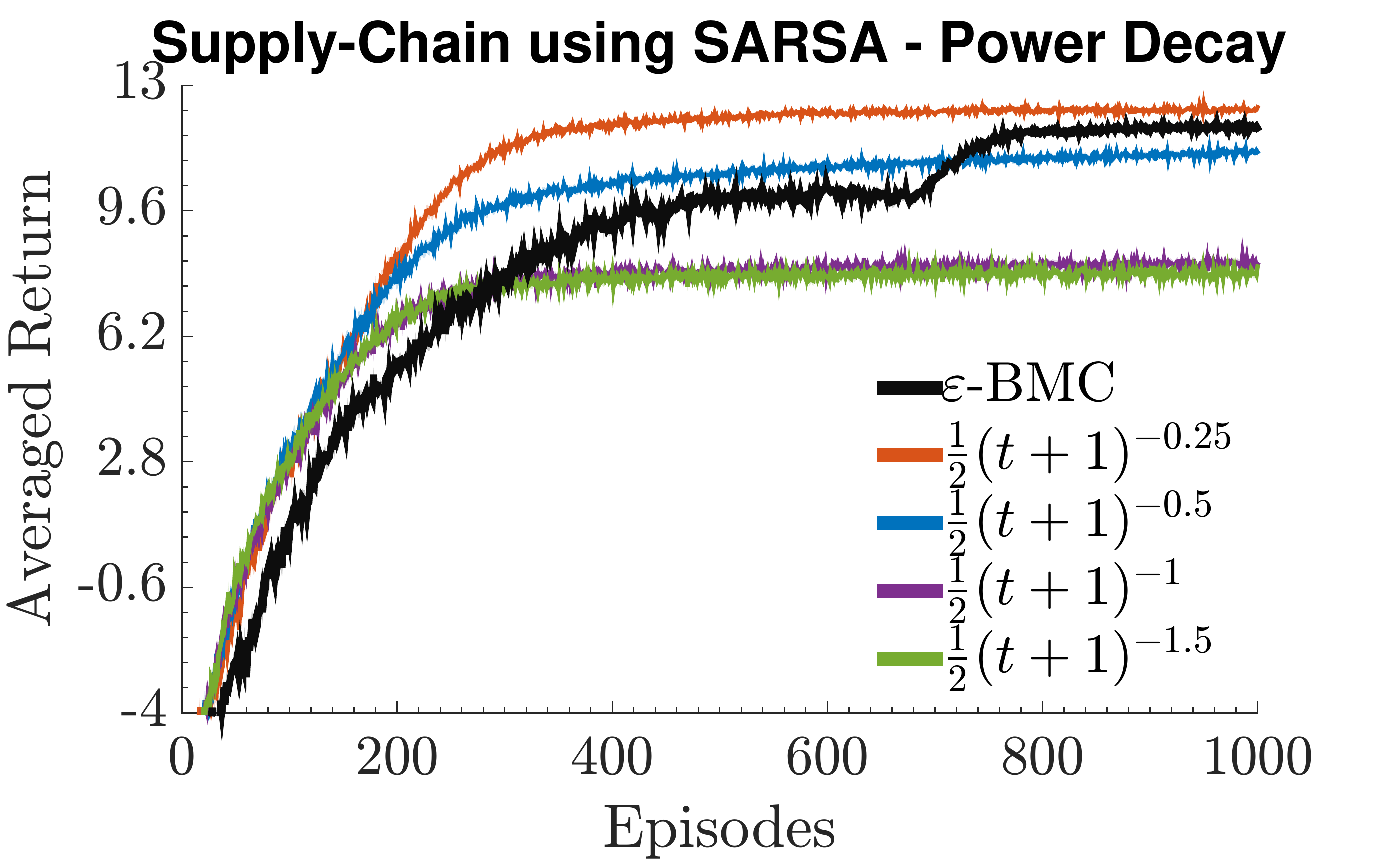} 
    \includegraphics[width=0.33\linewidth]{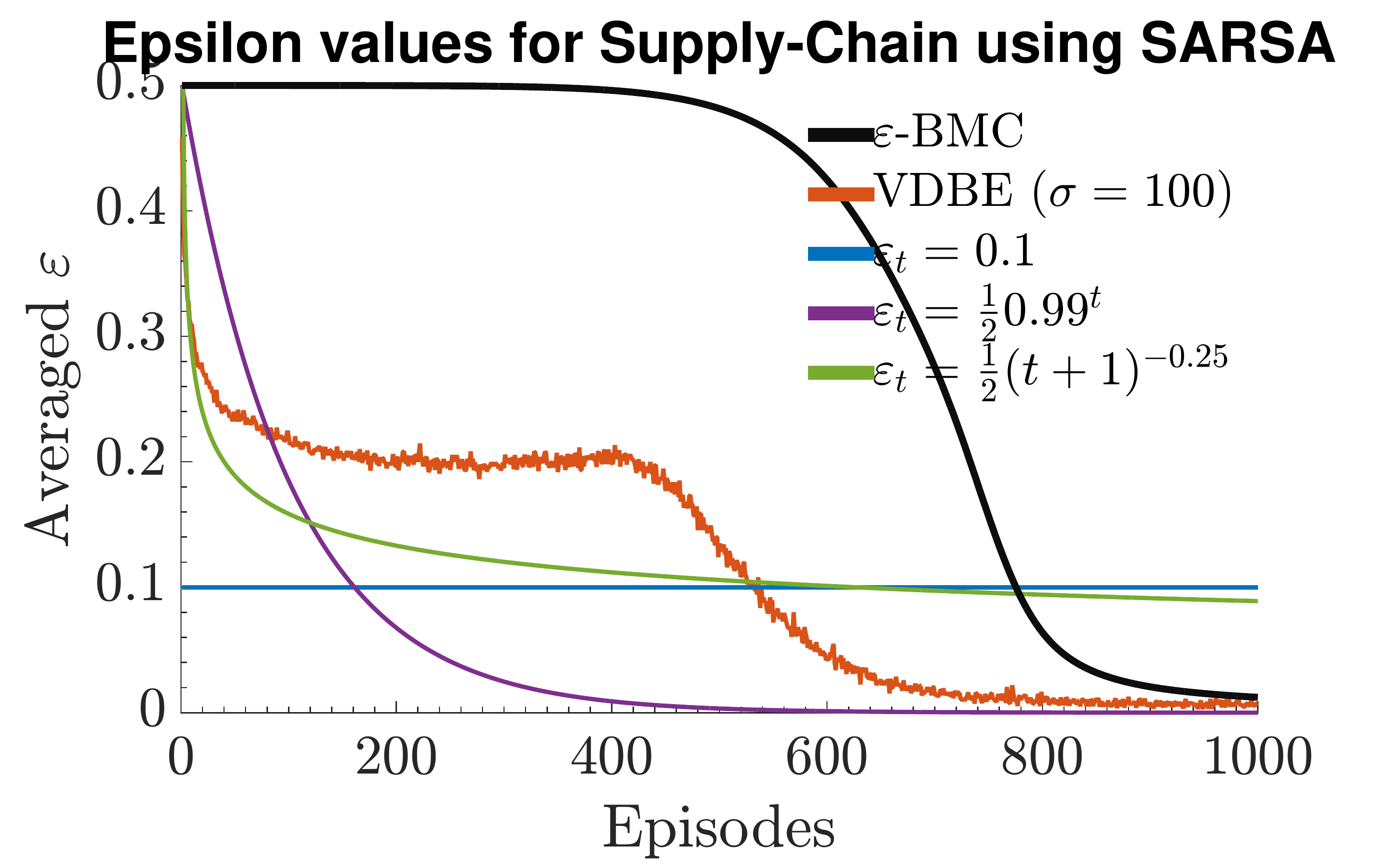}
    \caption{Average performance (return) on the supply-chain domain using expected SARSA.}
\label{fig:inventory_sarsa}
\end{figure*}

\begin{figure*}[t!]
    \centering
    \includegraphics[width=0.33\linewidth]{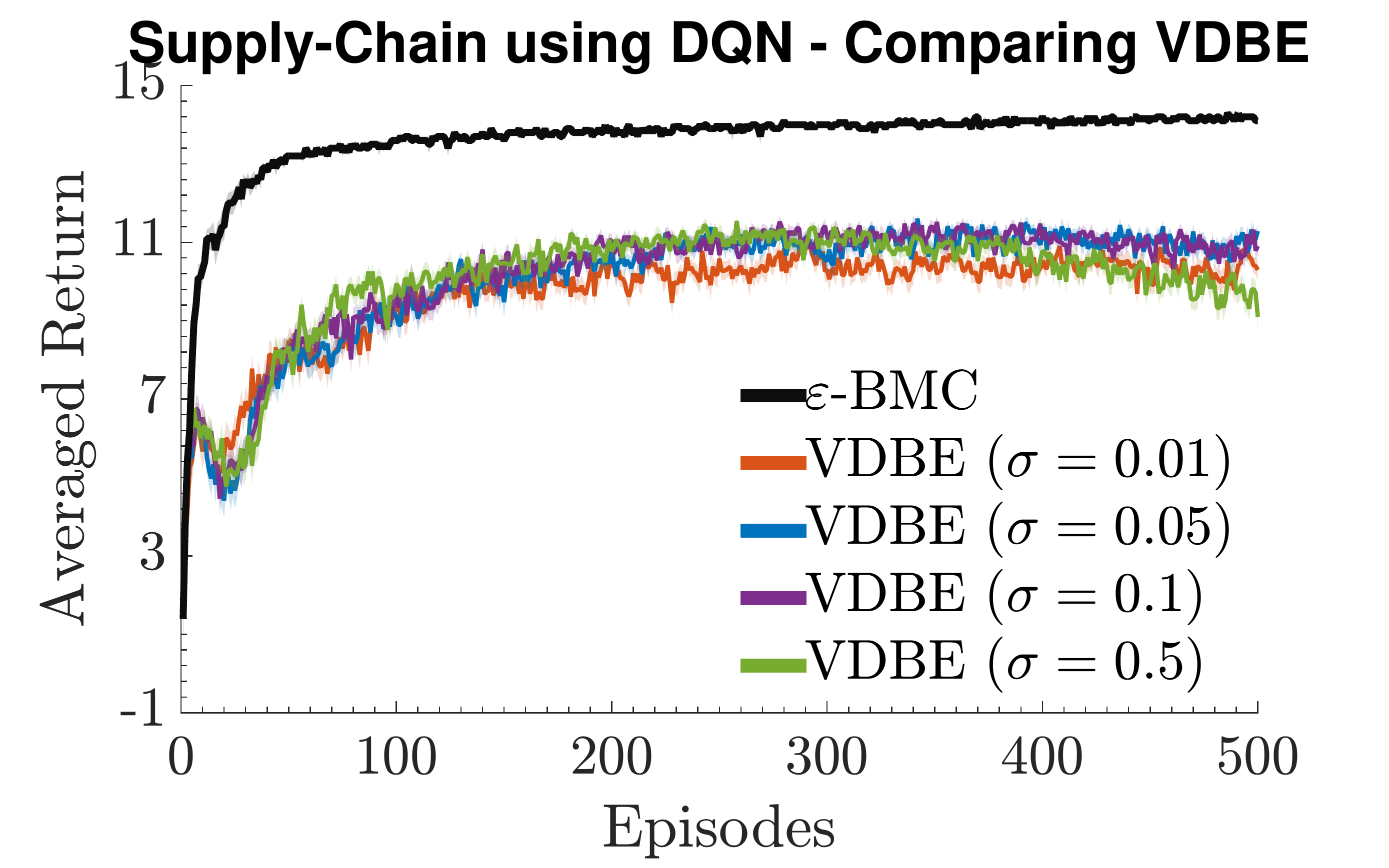}
    \includegraphics[width=0.33\linewidth]{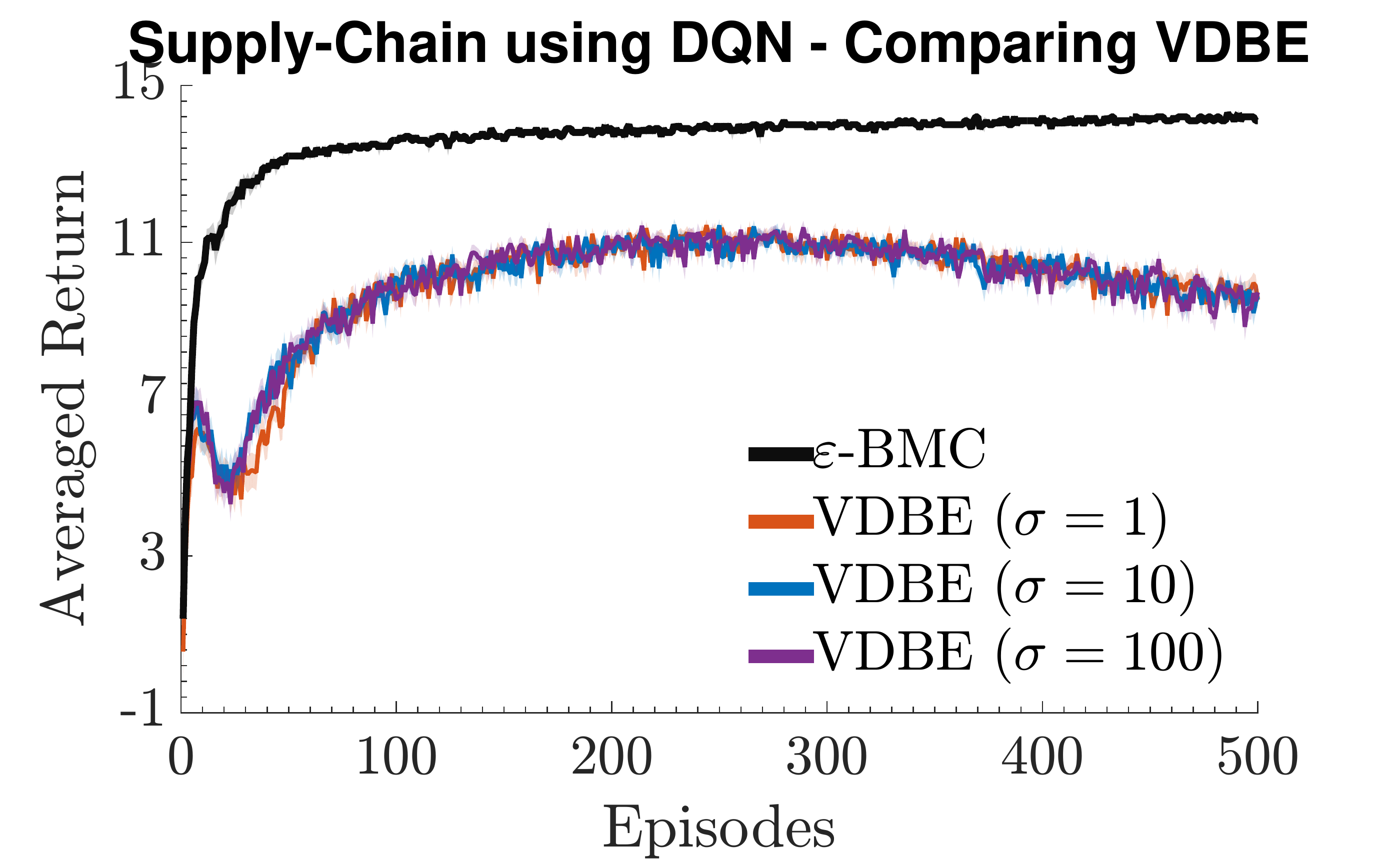}
    \includegraphics[width=0.33\linewidth]{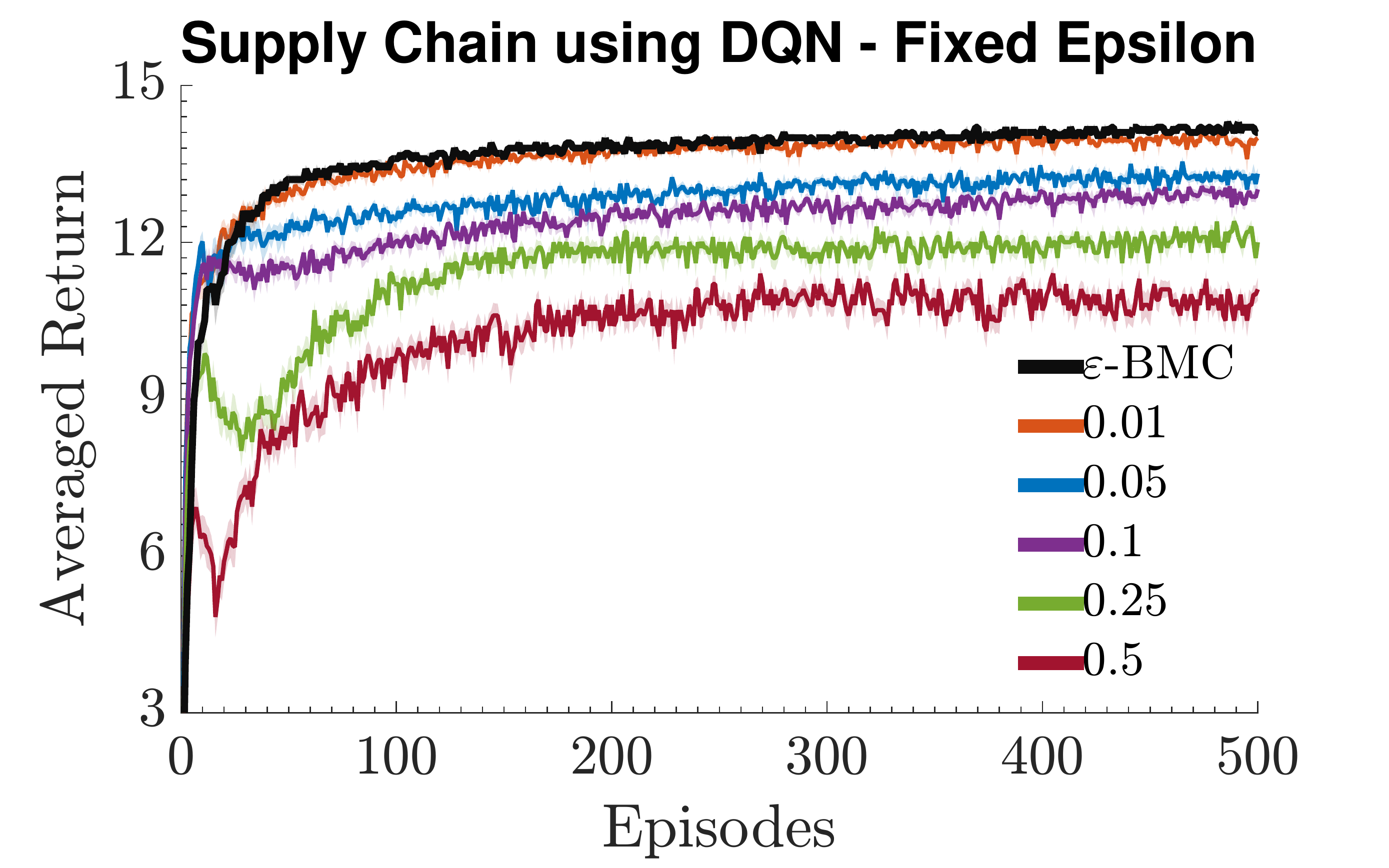} 
    \includegraphics[width=0.33\linewidth]{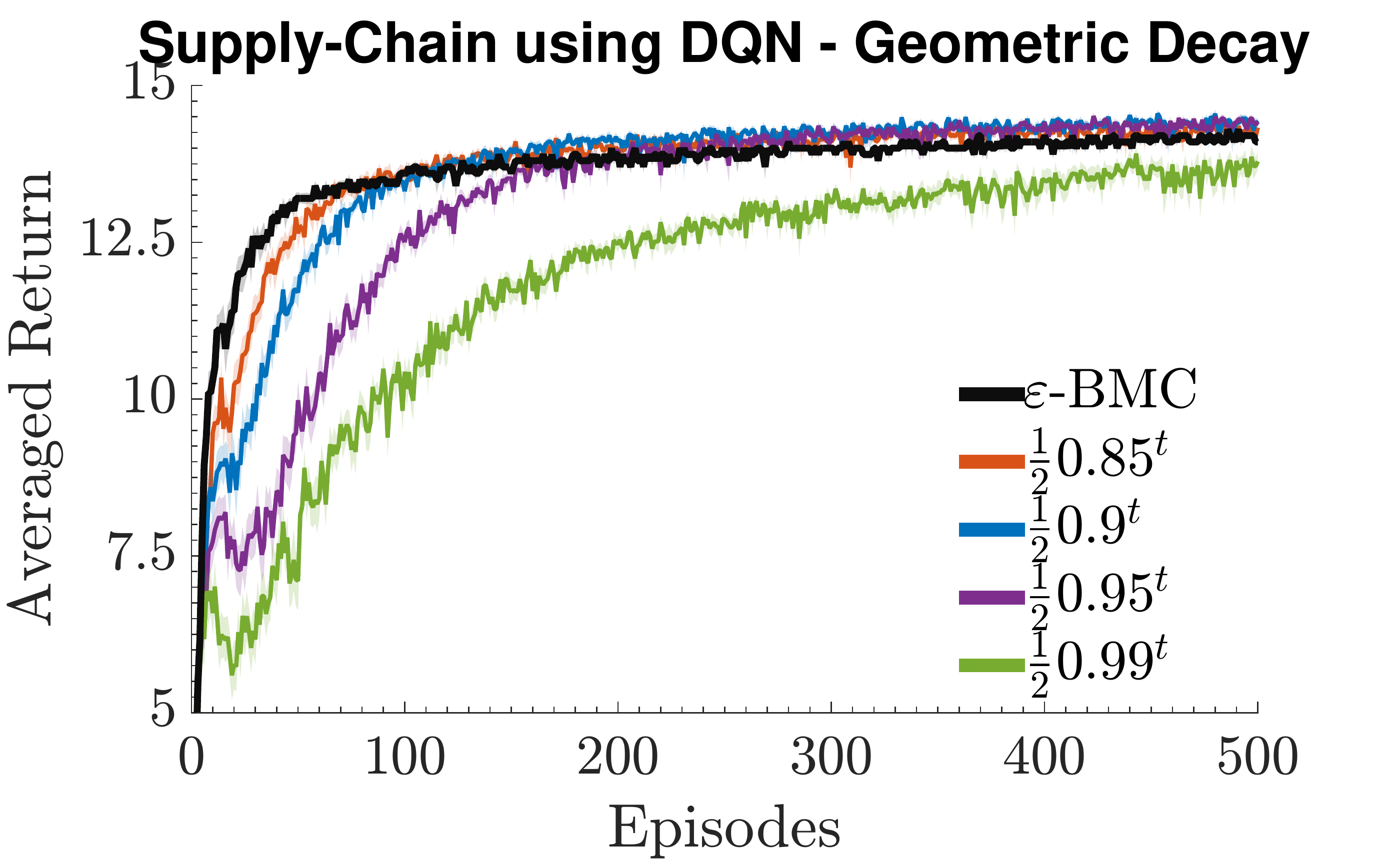}
    \includegraphics[width=0.33\linewidth]{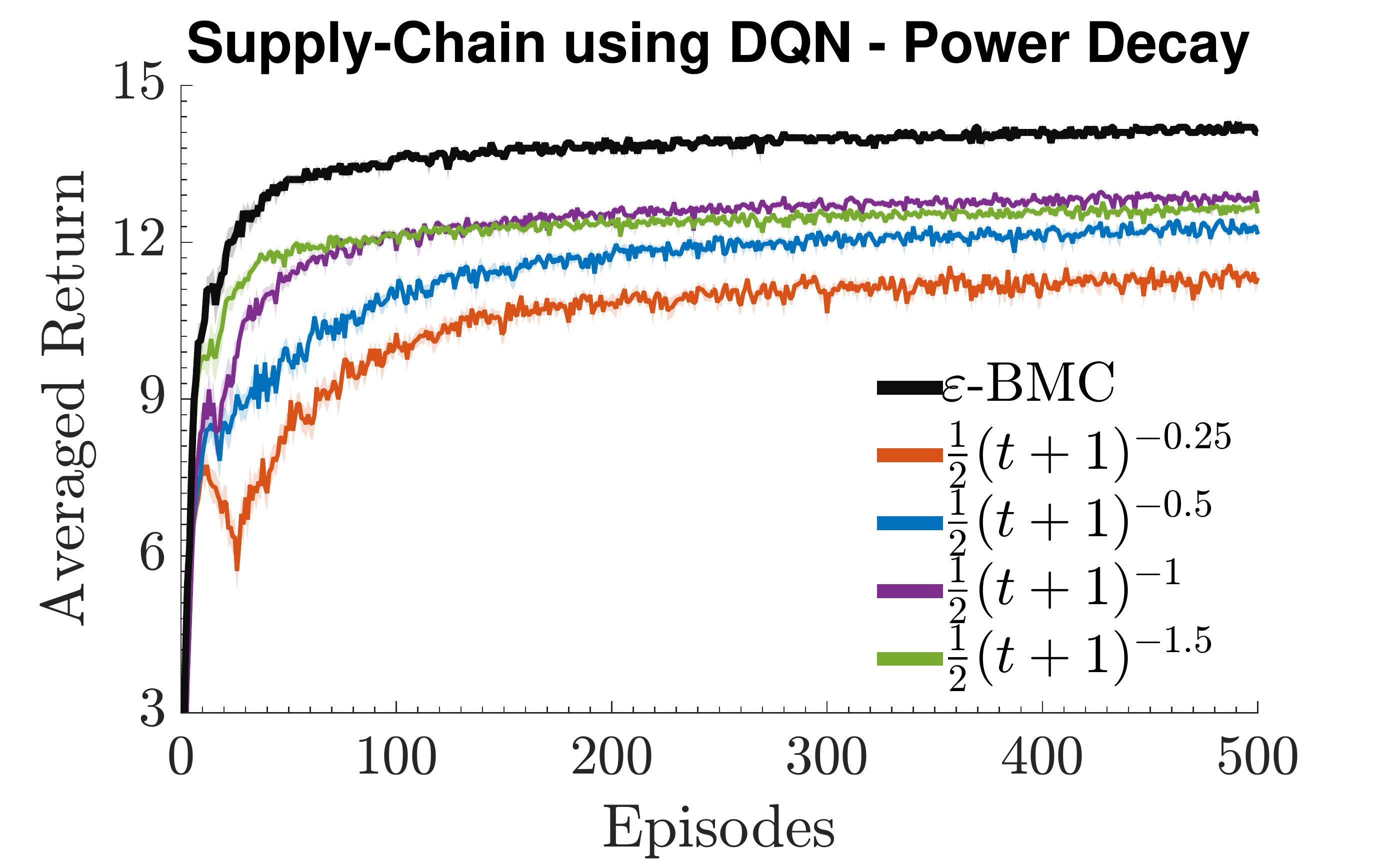} 
    \includegraphics[width=0.33\linewidth]{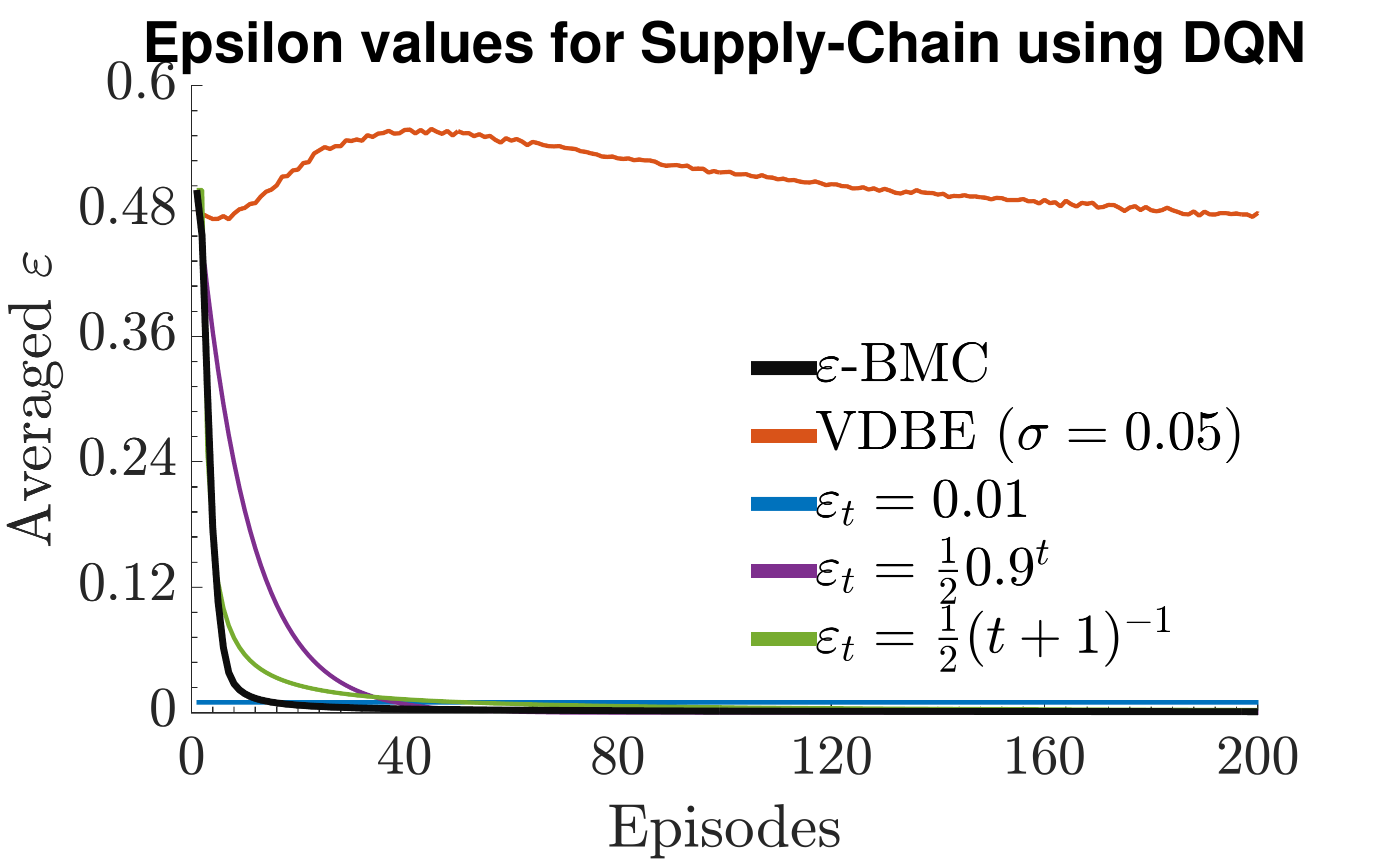}
    \caption{Average performance (return) on the supply-chain domain using deep Q-learning.}
\label{fig:inventory_dqn}
\end{figure*}

\subsection{DISCUSSION}

Overall, we see that $\varepsilon$-\texttt{BMC} consistently outperformed all other types of $\varepsilon$ annealing strategies, including VDBE, or performed similarly. However, $\varepsilon$-\texttt{BMC} converged slightly later than VDBE on the grid-world domain and the fixed annealing strategy $\varepsilon_t = \frac{1}{2} (t+1)^{-0.25}$ on the supply-chain problem, using tabular expected SARSA. However, in the former case, $\varepsilon$-\texttt{BMC} outperformed all fixed tuning strategies, and in the latter case, it outperformed VDBE by a large margin. These observations are related to the speed of convergence; asymptotically, $\varepsilon$-\texttt{BMC} approached the performance of the best policy that was attained (for grid-world this is indeed the optimal policy).

While it performed well on the simple grid-world domain, VDBE performed considerably worse than $\varepsilon$-\texttt{BMC} on the more complex supply-chain problem. We believe that the Bayesian approach of $\varepsilon$-\texttt{BMC} smooths out the noise in the return signals better than VDBE and other ad-hoc approaches for adapting $\varepsilon$. This also suggests why our algorithm performed better on DQN. 

Furthermore, we see that no single family of annealing strategies worked consistently well across all domains and algorithms. For instance, geometric decay strategies worked well on the grid-world domain, while performing poorly on the supply-chain problem using tabular SARSA. The power decay strategies worked well on the supply-chain problem using tabular SARSA, but failed to match the performance of other strategies when switching to DQN. Also, the performance of VDBE was highly sensitive to the choice of the $\sigma$ parameter. A lower value of $\sigma$ worked well for grid-world and cart-pole, but higher values of $\sigma$ worked better for supply-chain. The performance of $\varepsilon$-\texttt{BMC} was relatively insensitive to the choice of prior parameters for $\mu$ and $\tau$ ($a_0, b_0, \mu_0, \tau_0$), so we were able to use the same values in all our experiments. However, unsurprisingly, it was more sensitive to the strength of the prior on $\varepsilon$ ($\alpha_0$, $\beta_0$). Since we can always set $\alpha_0 \approx \beta_0$, this effectively reduces to the problem of selecting a single parameter that controls the strength of the prior on $\varepsilon$. This is considerably easier to do than to select both a good annealing method \emph{and} the tuning parameter(s).

\section{CONCLUSION}
\label{sec:end}

In this paper, we proposed a novel Bayesian approach to solve the exploration-exploitation problem in general model-free reinforcement learning, in the form of an adaptive epsilon-greedy policy. Our novel algorithm, $\varepsilon$-\texttt{BMC}, is a novel approach for tuning the $\varepsilon$ parameter automatically from return observations based on Bayesian model combination and approximate moment-matching based inference. It was argued to be general, efficient, robust, and theoretically grounded, and was shown empirically to outperform fixed annealing schedules for $\varepsilon$ and even a state-of-the-art $\varepsilon$ adaptation scheme. 

In future work, it would be interesting to evaluate the performance of $\varepsilon$-\texttt{BMC} combined with Boltzmann exploration \citep{tokic2011value}, as well as the state-dependent version. We believe that it is possible to obtain a Bayesian interpretation of VDBE by placing priors over the Bellman errors and updating them using data, but we have not investigated this approach. It would also be interesting to extend our approach to handle options.

\subsubsection*{Acknowledgements}

We thank the reviewers for their insightful comments and suggestions. 

\subsubsection*{References}

\bibliographystyle{abbrvnat}
\bibliography{references}

\ifappendix
    \appendix
    \newpage

    \section*{Supplementary Material}
    
    Tables~\ref{tab:sarsa} and~\ref{tab:dqn} describe parameter settings used in the experimentation for SARSA and DQN, respectively.
    
    
    \begin{table*}[b]
    \caption{Parameter settings for the tabular expected SARSA algorithm.}
    \label{tab:sarsa}
    \centering
    \begin{tabular}{lllll}
    \multicolumn{1}{c}{\bf PARAMETER} 
    &\multicolumn{1}{c}{\bf DESCRIPTION}
    &\multicolumn{1}{c}{\bf GRID-WORLD}
    &\multicolumn{1}{c}{\bf CART-POLE}
    &\multicolumn{1}{c}{\bf SUPPLY-CHAIN}\\
    \hline \\
    & Table initialization & uniform on [0, 0.1] & zeros & uniform on [0, 0.1] \\
    $\eta_t$ & Learning rate ($t$ episode \#) & 0.7 & $\max\left\lbrace \frac{1}{2} 0.99^{t}, 0.01 \right\rbrace$ & 0.6 \\
    $T$ & Max. episode length & 200 & 200 & 200 \\
    $\mu_0$ & Prior parameter in (\ref{eqn:gaussian-gamma-prior}) & 0 & 0 & 0 \\
    $\tau_0$ & Prior parameter in (\ref{eqn:gaussian-gamma-prior}) & 1 & 1 & 1\\
    $a_0$ & Prior parameter in (\ref{eqn:gaussian-gamma-prior}) & 500 & 500 & 500 \\
    $b_0$ & Prior parameter in (\ref{eqn:gaussian-gamma-prior}) & 500 & 500 & 500 \\
    $\alpha_0$ & Prior parameter for $\varepsilon$ & 1 & 10 & 1000 \\
    $\beta_0$ & Prior parameter for $\varepsilon$ & 1 + 0.01 & 10 + 0.01 & 1000 + 0.01 \\
    \end{tabular}
    \end{table*}
    
    
    \begin{table*}[b]
    \caption{Parameter settings for the deep Q-learning algorithm.}
    \label{tab:dqn}
    \centering
    \begin{tabular}{lllll}
    \multicolumn{1}{c}{\bf PARAMETER} 
    &\multicolumn{1}{c}{\bf DESCRIPTION}
    &\multicolumn{1}{c}{\bf GRID-WORLD}
    &\multicolumn{1}{c}{\bf CART-POLE}
    &\multicolumn{1}{c}{\bf SUPPLY-CHAIN}\\
    \hline \\
    & Network initialization & Glorot uniform & Glorot uniform & Glorot uniform \\
    & Network topology & 16-25-25-4 & 4-12-12-2 & 102-100-100-100 \\
    $f$ & Hidden activation & ReLU & ReLU & ReLU \\
    & Regularization & none & L2($10^{-6}$) & none \\ 
    $\phi$ & State encoding & one-hot & none & one-hot \\
    $\eta_t$ & Learning rate & 0.001 & 0.0005 & 0.001 \\
    $N$ & Replay buffer size & 2000 & 2000 & 3000 \\
    $B$ & Batch size & 24 & 32 & 64 \\
    & Training epochs per batch & 5 & 3 & 2 \\
    $T$ & Max. episode length & 200 & 200 & 200 \\
    $\mu_0$ & Prior parameter in (\ref{eqn:gaussian-gamma-prior}) & 0 & 0 & 0 \\
    $\tau_0$ & Prior parameter in (\ref{eqn:gaussian-gamma-prior}) & 1 & 1 & 1 \\
    $a_0$ & Prior parameter in (\ref{eqn:gaussian-gamma-prior}) & 500 & 500 & 500 \\
    $b_0$ & Prior parameter in (\ref{eqn:gaussian-gamma-prior}) & 500 & 500 & 500 \\
    $\alpha_0$ & Prior parameter for $\varepsilon$ & 1 & 5 & 25 \\
    $\beta_0$ & Prior parameter for $\varepsilon$ & 1 + 0.01 & 5 + 0.01 & 25 + 0.01 \\
    \end{tabular}
    \end{table*}
\fi

\end{document}